\documentclass[pdflatex,iicol]{sn-jnl}
\usepackage{amsfonts}
\usepackage{braket}
\usepackage{xcolor}
\usepackage{multirow}
\usepackage{qcircuit}
\usepackage{amsmath}
\usepackage{enumerate}
\usepackage{amsthm,mathrsfs}
\usepackage{tikz}
\usetikzlibrary{matrix,decorations.pathreplacing}
\usepackage{mleftright}
\usepackage{amssymb}
\usepackage{algorithm}
\usepackage{algpseudocode}
\algrenewcommand\algorithmicrequire{\textbf{Input:}}
\algrenewcommand\algorithmicensure{\textbf{Output:}}

\usepackage[capitalize,nameinlink,noabbrev]{cleveref}
\hypersetup{
   colorlinks,
   linkcolor={red!100!black},
   citecolor={green!100!black},
}

\usepackage{color, colortbl}
\definecolor{Gray}{gray}{0.85}

\usepackage[sort&compress]{natbib}
\Crefname{equation}{Eq.}{Eqs.}
\usepackage{mathabx}
\usepackage{soul}
\usepackage{geometry}
\usepackage{array}
\usepackage{comment}
\newtheorem{theorem}{Theorem}
\usepackage{subcaption}

\newtheorem{lemma}[theorem]{Lemma}

\newtheorem{fact}[theorem]{Fact}
\newtheorem{result}{Result}

\newtheorem{corollary}[theorem]{Corollary}

\usepackage{mathtools}

\DeclareMathOperator*{\argmax}{arg\,max}

\usepackage{geometry}
\def\be{\begin{eqnarray}}
\def\ee{\end{eqnarray}}
\usepackage{enumitem}

\makeatletter
\newcounter{phase}[algorithm]
\newlength{\phaserulewidth}
\newcommand{\setphaserulewidth}{\setlength{\phaserulewidth}}
\newcommand{\phase}[1]{%
  \vspace{-0.5ex}
  \Statex\leavevmode\llap{\rule{\dimexpr\labelwidth+\labelsep}{\phaserulewidth}}\rule{\linewidth}{\phaserulewidth}
  \Statex\strut\refstepcounter{phase}\textbf{#1}
  \vspace{-1.5ex}\Statex\leavevmode\llap{\rule{\dimexpr\labelwidth+\labelsep}{\phaserulewidth}}\rule{\linewidth}{\phaserulewidth}}
\makeatother

\setphaserulewidth{.7pt}

\newcommand{\poly}{\operatorname{poly}}

\usepackage{graphicx}

\usepackage{scalerel}
\usepackage{stackengine,wasysym}

\algnewcommand\algorithmicforeach{\textbf{for all:}}
\algnewcommand\ForEach{\item[ \algorithmicforeach]}

\allowdisplaybreaks



\begin{document}

\title[A Bit of Freedom Goes a Long Way: Classical and Quantum Algorithms for Reinforcement Learning under a Generative Model]{A Bit of Freedom Goes a Long Way: Classical and Quantum Algorithms for Reinforcement Learning under a Generative Model}

\author[1]{\fnm{Andris} \sur{Ambainis}}\email{andris.ambainis@lu.lv}

\author[2]{\fnm{Joao F.} \sur{Doriguello}}\email{doriguello@renyi.hu}

\author[1]{\fnm{Debbie} \sur{Lim}}\email{limhueychih@gmail.com}

\affil[1]{\orgdiv{Center for Quantum Computer Science, Faculty of Computing}, \orgname{University of Latvia}, \state{Riga}, \country{Latvia}}

\affil[2]{\orgname{HUN-REN Alfr\'ed R\'enyi Institute of Mathematics}, \state{Budapest}, \country{Hungary}}



\abstract{We propose novel classical and quantum online algorithms for learning finite- and infinite-horizon Markov Decision Processes (MDPs). Our algorithms are based on a hybrid online-offline reinforcement learning model wherein the agent can, from time to time, freely interact with the environment in a generative sampling fashion, i.e., by having access to a ``simulator''. By employing known classical and new quantum algorithms for approximating optimal policies under a generative model within our learning algorithms, we show that it is possible to avoid several paradigms from RL like ``optimism in the face of uncertainty'' and ``posterior sampling'' and instead compute and use optimal policies directly, which yields better regret bounds compared to previous works. Our quantum algorithms obtain regret bounds which only a $\operatorname{poly}\log{T}$ dependence on the number of time steps $T$, thus breaking the $O(\sqrt{T})$ classical barrier. Our infinite-horizon discounted regret bound is brand new, while in the finite- and infinite-horizon undiscounted settings, our results match the time dependence of some prior quantum works, but with improved dependence on other parameters like state space size $S$ and action space size $A$.}




\maketitle

\section{Introduction}

Reinforcement learning (RL)~\cite{sutton1998reinforcement} is a subfield of machine learning that studies how an agent can properly interact with a dynamical environment in order to maximise some type of reward. Markov Decision Processes (MDPs)~\cite{puterman2014markov} serve as the most commonly used framework for modeling such agent-environment interactions. Relevant to our work are time-independent MDPs described by a tuple $\langle \mathcal{S},\mathcal{A},p,r\rangle$, where the finite state space $\mathcal{S}$ of size $S$ is the set of possible states the environment can assume, the finite action space $\mathcal{A}$ of size $A$ is the set of possible actions the agent can choose, the reward function $r:\mathcal{S}\times\mathcal{A}\to[0,1]$ yields a reward to the agent when interacting with the environment, and the stochastic kernels $p(\cdot|s,a)$ denote the transition probability to another environment state given the current state-action pair $(s,a)$. At any (discrete) point in time, the environment is in some state $s_t\in\mathcal{S}$ and the agent must choose an action $a_t\in\mathcal{A}$, after which they receive a reward $r(s_t,a_t)$ and the environment randomly transitions to a new state $s_{t+1}$ according to $p(\cdot|s_t,a_t)$. The agent chooses an action through a \emph{policy}  $\pi = (\pi_t)_t$, which is a sequence of \emph{decision rules} $\pi_t$, which in turn are probability distribution over actions $a\in\mathcal{A}$ given $s\in\mathcal{S}$. 
The practical utility of MDPs lies in their ability to model decision-making in complex environments~\cite{aastrom1965optimal,hu2007markov,sato2010markov,bauerle2011markov,feinberg2012handbook,bennett2013artificial,chen2014distributed,steimle2017markov,natarajan2022planning}, with successful applications to many problems~\cite{sutton1998reinforcement,Szepesvari2010algorithms,bertsekas2012dynamic,bertsekas2022abstract}.

Parallel to RL developments, quantum machine learning~\cite{Biamonte2017quantum}, a subfield of quantum computation~\cite{nielsen2010quantum}, has been consolidated as an area of study that can provide quantum speedups to several traditional machine learning problems~\cite{lloyd2013quantum,rebentrost2014quantum,kerenidis2017quantum,kerenidis2019qmeans,doriguello2023you}, including RL~\cite{dong2008quantum}. It has been found that quantum algorithms are not only capable of achieving speedup in the time complexity of certain tasks, but also have the potential to be better ``learners'' than their classical counterparts in the online setting~\cite{ganguly2023quantum,zhong2023provably}.

In this work we study two fundamental problems related to MDPs: (i) approximating optimal policies under a generative model and (ii) online learning, both in classical and quantum settings. For that, we propose a hybrid online-offline RL model in which quantum computers are capable of offering an improved learning progress compared to previous works~\citep{auer2008near,bartlett2009regal,ortner2012online,lakshmanan2015improved,fruit2018efficient,ganguly2023quantum,zhong2023provably}.

\subsection{Preliminaries}

\textbf{Notations:} For $n\in\mathbb{N} := \{1,2,\dots\}$, let $[n]:=\{1, \dots, n\}$. Given a finite set $\mathcal{S}$, let $\mathscr{B}(\mathcal{S})$ be the space of all bounded Borel measurable real-valued functions on $\mathcal{S}$, which can be interpreted as $\mathbb{R}^S$ for $S = |\mathcal{S}|$, and let $\Delta(\mathcal{S})$ denote the probability simplex over $\mathcal{S}$. Given ${u}\in\mathscr{B}(\mathcal{S})$, its $\ell_1$-, $\ell_\infty$-norm, and span seminorm are $\Vert {u}\Vert_1 := \sum_{i\in\mathcal{S}} \vert u(i)\vert$, $\Vert {u}\Vert_\infty := \max_{i\in\mathcal{S}} \vert u(i)\vert$, and $\operatorname{sp}(u) := \max_{i\in\mathcal{S}}u(i) - \min_{i\in\mathcal{S}}u(i)$. We use $\mathbf{1}\in\mathscr{B}(\mathcal{S})$ to denote the all-ones function, $\mathbf{1}(s) = 1$. Unless mentioned otherwise, $\widetilde{O}(\cdot)$ hides polylogarithmic factors.

\textbf{Quantum Preliminaries:} Little background on quantum computation is needed for our paper and we refer the reader to~\cite{nielsen2010quantum} for more information. The quantum state of a quantum system is described by a unit vector from a Hilbert space denoted by the ket notation $|\cdot\rangle$. An $n$-qubit system is described by a unit vector in $\mathbb{C}^{2^n}$. The evolution of a quantum state $|\psi\rangle\in\mathbb{C}^{2^n}$ is described by a unitary operator $U\in\mathbb{C}^{2^n\times 2^n}$, $UU^\dagger = I$ where $U^\dagger$ is the Hermitian conjugate of $U$. 
We use $\bar 0$ to denote the all-zeros vector and $\ket{\bar 0}$ to denote the state $\ket{0}\otimes\cdots\otimes\ket{0}$ where the number of qubits is clear from context.

In this paper, we shall employ quantum oracles for functions and probability distributions. We say we have quantum access to $u\in\mathscr{B}(\mathcal{S})$ if we have access to the oracle $\mathcal{O}_u:|s\rangle|\bar{0}\rangle \mapsto |s\rangle|u(s)\rangle$ and its inverse, and we say we have quantum sampling access to $p\in\Delta(\mathcal{S})$ if we have access to the oracle $\mathcal{O}_p:|\bar{0}\rangle \mapsto \sum_{s\in\mathcal{S}} \sqrt{p(s)}|s\rangle$ and its inverse. Quantum access to a function is usually referred to as a quantum random access memory (QRAM)~\cite{giovannetti2008architectures,giovannetti2008quantum,jaques2023qram,allcock2023constant}. It is possible to build quantum access to  $u\in\mathscr{B}(\mathcal{S})$ in $O(S)$ time.

\textbf{MDP Preliminaries:} A policy is \emph{deterministic} if it is a sequence of stochastic kernels $\pi=(\pi_t)_{t}$ on $\mathcal A$ given $\mathcal S$ such that $\pi_t(a\vert s) = 1$ for some $a\in\mathcal{A}$. A policy is said to be \emph{stationary} if it is a constant sequence $\pi=(\pi_t)_{t}$ of decision rules on $\mathcal A$ given $\mathcal{S}$ such that $\pi_t = \pi_{t'}$ for all $t,t'$. Given a decision rule $d$, we employ the notation $d^\infty$ for the stationary policy $\pi = (d,d,\dots)$.

Let $\mathcal{L}_a:\mathscr{B}(\mathcal{S})\to \mathscr{B}(\mathcal{S})$ be the Bellman operator associated with action $a\in\mathcal{A}$ and $\mathcal{L}:\mathscr{B}(\mathcal{S})\to \mathscr{B}(\mathcal{S})$ be the optimal Bellman operator defined, $\forall u\in\mathscr{B}(\mathcal{S}), s\in\mathcal{S}$, as
\begin{align*}
     (\mathcal{L}_a u)(s) &:= r(s,a) + \sum_{s'\in\mathcal{S}} p(s'|s,a) u(s'),\\
     (\mathcal{L}u)(s) &:= \max_{a\in\mathcal{A}}\{(\mathcal{L}_a u)(s)\}.
\end{align*}
The operator $\mathcal{L}$ is monotonic, $u\leq v \implies \mathcal{L}u \leq \mathcal{L}v$, and non-expansive, $\operatorname{sp}(\mathcal{L}u - \mathcal{L}v) \leq \operatorname{sp}(u - v)$ and $\|\mathcal{L}u - \mathcal{L}v\|_\infty \leq \|u - v\|_\infty$ $\forall u,v\in\mathscr{B}(\mathcal{S})$.

There are several types of MDPs in the literature, the most common ones being \emph{finite-horizon} MDPs, \emph{infinite-horizon discounted} MDPs, and \emph{infinite-horizon undiscounted} MDPs~\cite{puterman2014markov}. In finite-horizon (also known as tabular) MDPs, the agent interacts with the environment for a finite pre-determined number of time steps $H$, which is called the \emph{horizon}. The standard criteria to evaluate the performance of the agent is the \emph{expected total reward} $V_1^\pi(s)$ over the horizon $H$ when executing policy $\pi$ with initial state $s\in\mathcal{S}$,
\begin{align*}
    V_1^\pi(s) := \mathbb{E}_{s_1=s}^\pi\left[\sum_{t=1}^{H} r(s_t,a_t)\right],
\end{align*}
where $\mathbb{E}_{s}^\pi$ denotes the expectation with respect to the probability distribution over trajectories $(s_1,a_1,s_2,a_2,\dots)$ with $s_1 = s$, $a_t\sim \pi_t(\cdot|s_t)$, and $s_{t+1}\sim p(\cdot|s_t,a_t)$. A policy $\pi^\varepsilon$ is $\varepsilon$-\emph{optimal} if $V_1^{\pi^\varepsilon}(s) \geq V_1^\pi(s) - \varepsilon$ for all $\pi$ and $s\in\mathcal{S}$, and the optimal total expected reward is $V_1^\ast(s) = \sup_\pi V_1^\pi(s)$. 

In infinite-horizon MDPs, both discounted and undiscounted, the agent can interact with the environment for an infinite amount of time steps, meaning that $H=\infty$. For discounted MDPs, the reward the agent receives at later time steps $t$ is discounted by a factor of $\gamma^t$, where $\gamma\in[0,1)$. The parameter $\Gamma := (1-\gamma)^{-1}$ is called the \emph{effective time horizon}. The standard criteria in evaluating the agent's performance is the \emph{discounted expected total reward} $V_1^{\pi,\gamma}(s)$ when executing policy $\pi$ with initial state $s\in\mathcal{S}$,
\begin{align*}
    V_1^{\pi,\gamma}(s) := \mathbb{E}_{s_1=s}^\pi\left[\sum_{t=1}^\infty \gamma^{t-1} r(s_t,a_t)\right].
\end{align*}
A policy $\pi^\varepsilon$ is $\varepsilon$-\emph{optimal} if $V_1^{\pi^\varepsilon,\gamma}(s) \geq V_1^{\pi,\gamma}(s) - \varepsilon$ for all $\pi$ and $s\in\mathcal{S}$, and the optimal total expected reward is $V_1^{\ast,\gamma}(s) = \sup_\pi V_1^{\pi,\gamma}(s)$. 

Regarding undiscounted MDPs, there is no discount factor, meaning that $\gamma=1$. Here the standard criteria to evaluate the performance of the agent is the \emph{average reward} (or gain) $g^\pi(s)$ when executing policy $\pi$ with initial state $s\in\mathcal{S}$,
\begin{align*}
    g^\pi(s) := \lim_{T\to\infty}\mathbb{E}_{s_1=s}^\pi\left[\frac{1}{T}\sum_{t=1}^{T} r(s_t,a_t) \right].
\end{align*}
A policy $\pi^\varepsilon$ is $\varepsilon$-optimal if $g^{\pi^\varepsilon}(s) \geq g^\pi(s) - \varepsilon$ for all $\pi$ and $s\in\mathcal{S}$, and the optimal average reward is $g^\ast(s) = \sup_\pi g^\pi(s)$.

There are several subtypes of infinite-horizon undiscounted MDPs, see~\cite[\S~8.3]{puterman2014markov}. In this work, we focus on \emph{weakly communicating}\footnote{This means that the states can be partitioned into two subsets, one in which all states are transient under any stationary policy, and the other in which any state is reachable from any other state under some stationary policy.} MDPs, for which the optimal gain $g^\ast(s)$ is state independent, i.e., $g^\ast(s) = g^\ast$ for all $s\in\mathcal{S}$, and any optimal policy $\pi^\ast$ has constant gain. Moreover, for any stationary policy $d^\infty$, there exists a function $h^{d^\infty}:\mathcal{S}\to\mathbb{R}$ called \emph{bias} that, together with the gain $g^{d^\infty}$, satisfy the Bellman equations
\begin{align*}
    g^{d^\infty}(s) &= \sum_{s'\in\mathcal{S}} \sum_{a\in\mathcal{A}} d(a|s') p(s'|s,a) g^{d^\infty}(s'),\\ 
    h^{d^\infty} &= \mathcal{L}_d h^{d^\infty} - g^{d^\infty}.
\end{align*}
There is an optimal stationary deterministic policy $(d^\ast)^\infty$ for which $(g^\ast, h^\ast) = (g^{(d^\ast)^\infty}, h^{(d^\ast)^\infty})$ satisfy the Bellman equation $h^\ast = \mathcal{L}h^\ast - g^\ast \mathbf{1}$.

\section{Computing optimal policies under a generative model}
\label{sec:computing_optimal_policies}

One of the main problems associated with MDPs is that of computing $\varepsilon$-optimal policies, for which several algorithms have been proposed~\cite{ bellman1966dynamic,watkins1992q,meyn1997policy, kearns1998finite, sutton1999policy,lagoudakis2003least,bertsekas2011approximate,azar2012dynamic,puterman2014markov,silver2014deterministic, schulman2015trust,  jang2019q}. The main interest is thus on the performance of the computed policy~\cite{sutton1998reinforcement,Kearns1999finite,Kearns2002near}. Several input access models can be considered for computing optimal policies. In this paper, we consider the so-called \emph{generative model}~\cite{kearns1998finite,Kearns2002sparse,kakade2003sample} where one has full knowledge of state and action spaces $\mathcal{S},\mathcal{A}$ and of the reward function $r:\mathcal{S}\times\mathcal{A}\to[0,1]$, but the transition probabilities $p(\cdot|s,a)$ can only be accessed through an oracle. In this scenario one is usually concerned with the sample complexity of employing such an oracle. In the classical setting, one has access to an oracle $\mathcal{C}_p$ that, on input $(s,a)\in\mathcal{S}\times\mathcal{A}$, returns $s'\in\mathcal{S}$ with probability $p(s'|s,a)$. 
In the quantum setting, one has access to an oracle $\mathcal{Q}_p$ (and its inverse) called quantum accessible-environment~\cite{wang2021quantum,wiedemann2022quantum, jerbi2022quantum, zhong2023provably} which is a unitary operator that acts, $\forall (s,a)\in\mathcal{S}\times\mathcal{A}$, as
\begin{align}\label{eq:quantum_sampling_oracle}
    \mathcal O_{p}: \ket{s}\ket{a}\ket{\bar 0}\rightarrow \sum_{s'\in\mathcal S} \sqrt{p(s'\vert s, a)} \ket{s}\ket{a}\ket{s'}. 
\end{align}
%

Here we revise known classical algorithms and propose novel quantum algorithms for computing optimal policies, which will serve as subroutines in learning MDPs in an online setting in \Cref{sec:intro_online_learning}.

\subsection{Finite-horizon MDPs} 

There is a long list of works that studied the classical query complexity of obtaining optimal policies~\cite{Kearns1999finite,Kearns2002sparse,GheshlaghiAzar2013,wang2017randomized,Sidford2018variance,sidford2018near,li2020breaking}. For finite-horizon MDPs, Sidford et al.~\cite{sidford2018near} obtained a sample-optimal algorithm that outputs an $\varepsilon$-optimal policy with $\widetilde{O}\big(\frac{H^3SA}{\varepsilon^2} \big)$ queries to $\mathcal{C}_p$, matching the lower bounds from~\cite{GheshlaghiAzar2013,sidford2018near} up to polylogarithmic factors. 
\begin{fact}[\cite{sidford2018near}]\label{fact:finite_horizon_classical}
    Let $\langle\mathcal{S},\mathcal{A},p,r,H\rangle$ be a finite-horizon MDP. There is a classical algorithm that outputs an $\varepsilon$-optimal policy with probability $1-\delta$ and query complexity (up to $\poly\log\log$ factors)
    \begin{align*}
        \widetilde{O}\left(\frac{H^3SA}{\varepsilon^2}\log\!\left(\frac{HSA}{\delta\varepsilon}\right)\log\!\left(\frac{H}{\varepsilon}\right)\right).
    \end{align*}
\end{fact}

On the other hand, the problem of computing optimal policies under a quantum generative model for finite-horizon MDPs is still mostly unexplored. 
As one of our main results, we design an improved quantum algorithm for computing optimal policies for tabular MDPs. 
\begin{result}\label{res:res1}
    Let $\langle\mathcal{S},\mathcal{A},p,r,H\rangle$ be a finite-horizon MDP. There is a quantum algorithm that outputs an $\varepsilon$-optimal policy with probability $1-\delta$ with query complexity (up to $\poly\log\log$ factors)
    \begin{align*}
        \widetilde{O}\left(\!\min\!\left\{\frac{H^{2.5}SA}{\varepsilon},\frac{H^3S\sqrt{A}}{\varepsilon} \right\}\log^2\!\left(\frac{HSA}{\delta\varepsilon} \right)\right).
    \end{align*} 
\end{result}
\begin{proof}
    See \Cref{sec:optimal_policies_finite-horizon}. 
\end{proof}

Our quantum algorithm is in reality the combination of two different quantum algorithms. The first one is a quantised version of the standard value iteration~\cite{puterman2014markov}, where one solves the equations
\begin{align*}
    u_t(s) = \max_{a\in\mathcal{A}} \left\{r(s,a) + \sum_{s'\in\mathcal{S}} p(s'|s,a)u_{t+1}(s') \right\}
\end{align*}
for all $s\in\mathcal{S}$ and $t=H,H-1,\dots,1$ with initial condition $u_{H+1}\equiv 0$. The optimal policy is then 
\begin{align*}
    \pi_t(s) \in \argmax_{a\in\mathcal{A}} \left\{r(s,a) + \sum_{s'\in\mathcal{S}} p(s'|s,a)u_{t+1}(s') \right\}.
\end{align*}
By employing the quantum mean estimation subroutine of Kothari and O'Donnell~\cite{Kothari2023mean}, we estimate $\sum_{s'\in\mathcal{S}} p(s'|s,a)u_{t+1}(s')$ with error $\frac{\varepsilon}{H}$ using $O\big(\frac{H^2}{\varepsilon}\big)$ queries to $\mathcal{Q}_p$. This quantum operation is nested into the quantum minimum finding subroutine of D\"urr and H\o{}yer~\cite{durr1996quantum} in order to find $u_t(s)$, which uses $O\big(\frac{H^2\sqrt{A}}{\varepsilon}\big)$ queries to $\mathcal{Q}_p$. By repeating this procedure for all $t\in[H]$ and $s\in\mathcal{S}$, we arrive at one of the final complexities in \Cref{res:res1}.

Our second quantum algorithm is a quantised version of the modern value iteration algorithms of~\cite{sidford2018near,li2020breaking}. Their backward induction algorithms work in \emph{epochs}: by starting the $k$-th epoch with functions $u^{(k-1)}_{1},\dots,u^{(k-1)}_{H}$ such that $0 \leq V^\ast_t(s) - u^{(k-1)}_{t}(s) \leq 2\epsilon_k$ for all $s\in\mathcal{S}$ and $t\in[H]$, their algorithms produce $u^{(k)}_{1},\dots,u^{(k)}_{H}$ such that $0 \leq V_t^\ast(s) - u^{(k)}_{t}(s) \leq \epsilon_k$ by the end of the $k$-th epoch, thus halving the initial error. By letting $u^{(0)}_1, \dots, u^{(0)}_H \equiv 0$ at the first epoch and noticing that $V_t^\ast(s) \leq H$, only $O\big({\log}\frac{H}{\varepsilon}\big)$ epochs are needed for a final $\varepsilon$-approximation.

The algorithm from \cite{sidford2018near} uses three crucial techniques: \emph{monotonicity}, \emph{variance reduction}, and \emph{total-variance} techniques. The monotonicity technique means maintaining the monotonicity condition $u^{(k)}_t(s) \leq V_t^{\pi^{(k)}}(s) \leq V_t^\ast(s)$, which yields an $\epsilon$-optimal policy, otherwise an $\epsilon$-optimal value function can yield an $2\epsilon H$-optimal greedy policy in the worst case~\cite{Singh1994upper,bertsekas2012dynamic}.

Naively (like in the standard value iteration), at each time step $t$ of each epoch $k$ one could estimate $\sum_{s\in\mathcal{S}}p(s'|s,a)u^{(k)}_t(s')$ up to additive error $\frac{\epsilon_k}{H}$ in order to obtain $u_{1}^{(k)},\dots,u_{H}^{(k)}$ with error $\epsilon_k$. Since $\|u^{(k)}_t\|_\infty \leq H$, by a Hoeffding bound $\widetilde{O}\big(\frac{H^4}{\epsilon_k^2}\big)$ samples would suffice for each $t\in[H]$, to a total of $\widetilde{O}\big(\frac{H^5}{\epsilon_k^2}\big)$ samples for all $t$. The variance reduction technique from~\cite{sidford2018near} rewrites the standard backward induction iteration as
\begin{align*}
    u_{t}^{(k)}(s) &\gets \max_{a\in\mathcal{A}}\bigg\{r(s,a) + \sum_{s'\in\mathcal{S}}p(s'|s,a) u^{(k-1)}_{t+1}(s')\\
    &+ \sum_{s'\in\mathcal{S}}p(s'|s,a)(u^{(k)}_{t+1}(s') - u^{(k-1)}_{t+1}(s'))\bigg\}.
\end{align*}
The main idea is that $\sum_{s'\in\mathcal{S}}p(s'|s,a) u^{(k-1)}_{t}(s')$ for all $t\in[H]$ can be computed at the beginning of the epoch using the same batch of $\widetilde{O}\big(\frac{H^4}{\epsilon_k^2}\big)$ samples, which saves a factor of $H$. Regarding the quantity $\sum_{s'\in\mathcal{S}}p(s'|s,a)(u^{(k)}_{t+1}(s') - u^{(k-1)}_{t+1}(s'))$, since $\|u^{(k)}_{t+1} - u^{(k-1)}_{t+1}\|_\infty \leq \epsilon_k$, it can be approximated up to error $\frac{\epsilon_k}{2H}$ using $\widetilde{O}(H^2)$ samples for each $t\in[H]$, leading to a total of $\widetilde{O}(H^3)$ samples.

Finally, in order to reduce the sample complexity dependence on $H$ down to $O(H^3)$, \cite{sidford2018near} show that the true error accumulates as $\sqrt{H^3/m}$ given $m$ samples, much less than the naive sum of estimation errors at each time step. This means that one does not require the update error to be $\frac{\epsilon_k}{H}$ at each time step and $m = O\big(\frac{H^3}{\epsilon_k^2}\big)$ samples suffices for a total error $\epsilon_k$.

In our quantum algorithm, the quantities $\sum_{s'\in\mathcal{S}}p(s'|s,a)(u^{(k)}_{t+1}(s') - u^{(k-1)}_{t+1}(s'))$ are estimated using the aforementioned subroutine of Kothari and O'Donnell~\cite{Kothari2023mean} with $\widetilde{O}(H)$ queries to the oracle $\mathcal{Q}_p$, for a total of $\widetilde{O}(H^2)$ queries. Similarly, the quantities $\sum_{s'\in\mathcal{S}}p(s'|s,a) u^{(k-1)}_{t}(s')$ are estimated using the quantum mean estimation from~\cite{Kothari2023mean}. However, unlike the classical case where the same batch of samples can be reused to approximate all these quantities, in the quantum setting we must start each calculation anew, which incurs an extra $H$ factor in the query complexity and hinders a full quadratic advantage in $H$.

We remark that the concurrent work of Luo et al.~\cite{luo2025quantum}, which appeared online shortly before ours, took the same approach and arrived at the same query complexity as \Cref{res:res1}. Although their result applies to the more general case of time-dependent MDPs, it is straightforward to generalise \Cref{res:res1} to such a scenario. We also mention that~\cite{luo2025quantum} proved the quantum query lower bound $\Omega\big(\frac{H^{1.5}S\sqrt{A}}{\varepsilon}\big)$. Hence our and Luo et al.~\cite{luo2025quantum} result might not be tight and could be improved. 

\subsection{Infinite-Horizon Discounted MDPs}

The scenario for computing optimal policies for infinite-horizon discounted MDPs have a similar history compared to finite-horizon MDPs. A long list of works~\cite{kearns1998finite,GheshlaghiAzar2013,wang2017randomized,Sidford2018variance} slowly improved the classical complexity until Sidford et al.~\cite{sidford2018near} obtained a sample-optimal algorithm for outputting an $\varepsilon$-optimal policy using $\widetilde{O}\big(\frac{\Gamma^3SA}{\varepsilon^2}\big)$ queries to $\mathcal{C}_p$, matching the lower bounds of~\cite{GheshlaghiAzar2013} up to polylogarithmic factors. The algorithm of Sidford et al.~\cite{sidford2018near} is very similar to sample-optimal algorithm for finite-horizon MDPs explained in the previous section (the same work proposed both algorithms), so we shall not dwell on it. The main difference between both algorithms is that for infinite-horizon discounted MDPs, one needs only to compute a single quantity $\sum_{s'\in\mathcal{S}} p(s'|s,a) u^{(k)}_0(s')$ at the beginning of each epoch, instead of $H$ different ones as for finite-horizon MDPs.

\begin{fact}[\cite{sidford2018near}]\label{fact:discounted_classical}
    Let $\langle\mathcal{S},\mathcal{A},p,r,\gamma\rangle$ be an infinite-horizon discounted MDP. There is a classical algorithm that outputs an $\varepsilon$-optimal policy with probability $1-\delta$ and query complexity (up to $\poly\log\log$ factors)
    \begin{align*}
        \widetilde{O}\left(\frac{\Gamma^3SA}{\varepsilon^2}\log\!\left(\frac{\Gamma SA}{\delta\varepsilon}\right)\log\!\left(\frac{\Gamma}{\varepsilon}\right)\right).
    \end{align*}
\end{fact}

In the quantum setting, there are only a few works related to the problem of computing optimal policies for infinite-horizon discounted MDPs~\cite{wiedemann2022quantum,wang2021quantum,jerbi2022quantum,cherrat2023quantum}. 
Within the generative setting, we are only aware of the work of Wang et al.~\cite{wang2021quantum}, who quantised the standard value iteration~\cite{puterman2014markov} and the modern value iteration algorithms from~\cite{sidford2018near,li2020breaking}, our quantum algorithm for finite-horizon MDPs thus being inspired by their approach. Their quantum standard value iteration algorithm nests an old version of quantum mean estimation~\cite{montanaro2015quant} within quantum minimum finding~\cite{durr1996quantum}, while their quantum version of modern value iteration employs solely quantum mean estimation~\cite{montanaro2015quant}. 
\begin{fact}[\cite{wang2021quantum}]\label{fact:discounted_quantum}
    Let $\langle\mathcal{S},\mathcal{A},p,r,\gamma\rangle$ be an infinite-horizon discounted MDP. There is a quantum algorithm that outputs an $\varepsilon$-optimal policy with probability $1-\delta$ and query complexity (up to $\poly\log\log$ factors)
    \begin{align*}
        \widetilde{O}\left({\min}\left\{\frac{\Gamma^{1.5}SA}{\varepsilon},\frac{\Gamma^3S\sqrt{A}}{\varepsilon}\right\}\log^4\left(\frac{\Gamma S A}{\delta\varepsilon}\right)\right).
    \end{align*}
\end{fact}

\subsection{Infinite-Horizon Undiscounted MDPs}

The list of classical algorithms for computing optimal policies for infinite-horizon undiscounted MDPs is as extensive as in the previous settings~\cite{wang2017primal,jin2020efficiently,jin2021towards,wang2022near,zhang2023sharper,li2024stochastic,wang2024optimal}, the best complexity being $\widetilde{O}\big(\frac{\Lambda SA}{\varepsilon^2}\big)$ due to Zureck and Chen~\cite{zurek2023span}, matching the lower bounds of~\cite{jin2021towards,wang2022near}. Here $\operatorname{sp}(h^\ast) \leq \Lambda$ is an upper-bound on the span of the optimal bias function $h^\ast(s)$, which is always finite under the mild assumption of a weakly communicating MDP used in several RL works~\cite{bartlett2009regal,ortner2012online,lakshmanan2015improved,fruit2018efficient}.
\begin{fact}[{\cite{zurek2023span}}] \label{fact:undiscounted_classical}
    Let $\langle \mathcal{S},\mathcal{A},p,r\rangle$ be an infinite-horizon undiscounted average-reward weakly communicating MDP with $\operatorname{sp}(h^\ast) \leq \Lambda$. There is a classical algorithm that outputs an $\varepsilon$-optimal policy with probability $1-\delta$ and query complexity (up to $\poly\log\log$ factors)
    \begin{align*}
        \widetilde{O}\left(\frac{\Lambda SA}{\varepsilon^2}\log\left(\frac{SA}{\delta\varepsilon}\right) \right).
    \end{align*}
\end{fact}

On the other hand, we are not aware of any related quantum works for this problem. As another contribution, we propose a quantum algorithm for computing optimal policies for infinite-horizon undiscounted MDPs. For such, we further assume the optimal Bellman operator $\mathcal{L}$ is a $J$-stage $\nu$-span contraction for $J\in\mathbb{N}$ and $\nu\in[0,1)$, meaning $\operatorname{sp}(\mathcal{L}^J u - \mathcal{L}^J  v) \leq  \nu \operatorname{sp}(u - v)$ $\forall u,v \!\in\!\mathscr{B}(\mathcal{S})$.
\begin{result}\label{res:res2}
    Let $\langle \mathcal{S},\mathcal{A},p,r\rangle$ be an infinite-horizon undiscounted average-reward weakly communicating MDP with $\operatorname{sp}(h^\ast) \leq \Lambda$. Assume that its associated optimal Bellman operator $\mathcal{L}$ is a $J$-stage $\nu$-span contraction for constant $J\in\mathbb{N}$ and $\nu\in[0,1)$. There is a quantum algorithm that outputs an $\varepsilon$-optimal policy with probability $1-\delta$ and query complexity (up to $\poly\log\log$ factors)
    \begin{align*}
        \widetilde{O}\left(\frac{\Lambda S\sqrt{A}}{\varepsilon}\log\left(\frac{1}{\varepsilon}\right)\log^2\left(\frac{SA}{\delta}\right)\right).
    \end{align*}
\end{result}
\begin{proof}
    See \Cref{sec:value_iteration}.
\end{proof}
Similar to the finite-horizon and discounted settings, our quantum algorithm is a quantised version of the standard value iteration~\cite{puterman2014markov} where $\sum_{s'\in\mathcal{S}} p(s'|s,a) u_t(s')$ is approximated using quantum mean estimation~\cite{Kothari2023mean} and the maximum over $a\in\mathcal{A}$ is computed using quantum minimum finding~\cite{durr1996quantum}. The span-contraction assumption on $\mathcal{L}$ guarantees that value iteration converges to an approximate optimal solution after a finite number of iterations. Furthermore, throughout our algorithm, we maintain the condition $\operatorname{sp}(u_t) \leq 2\operatorname{sp}(h^\ast)$ over the iterated functions $(u_t)_t$.

Although not as mild an assumption as only requiring  $\operatorname{sp}(h^\ast) \leq \Lambda$ to be finite like in the classical algorithm of~\cite{zurek2023span} (\Cref{fact:discounted_classical}), assuming $\mathcal{L}$ to be a $J$-stage $\nu$-span contraction can be enforced under some conditions on the stochastic kernel $p$ (see~\citep[Theorems~8.5.2 \&~8.5.3]{puterman2014markov}) and has been used in previous works~\cite{puterman2014markov,fruit2018efficient}. We leave as future work weakening the assumptions of \Cref{res:res2}.

\section{Online learning of MDPs}
\label{sec:intro_online_learning}

As our main contribution, we propose improved classical and quantum algorithms for learning unknown MDPs in an online fashion. Here one is interested in the performance of the learning algorithm during the learning process, an area of study that finds applications in several topics~\cite{crammer2003online, ying2006online,liang2006fast,tekin2010online,li2014online, ouyang2017learning,aaronson2018online,lim2022quantum}. Interacting with an unknown MDP can be naturally framed as an online learning problem: at each time step $t\in\mathbb{N}$, the environment is at some state $s_t\in\mathcal{S}$ and the agent must choose an action $a_t\sim \pi_t(\cdot|s_t)$ according to some policy $\pi = (\pi_t)_t$ in order to receive as large a reward $r(s_t,a_t)$ as possible, after which the environment transitions to a new state $s_{t+1}\sim p(\cdot|s_t,a_t)$. However, since the agent does not know the underlying transition probabilities $p(\cdot|s,a)$, an exploration-exploitation trade-off arises: should the agent explore poorly understood states and actions in order to improve its understanding of the MDP and improve its future performance via better policies, or exploit its current knowledge to optimise short-term rewards~\cite{Kearns2002near}. 


The task of obtaining large rewards is usually reframed as minimising some measure of how far the agent is from being optimal~\cite{valiant1984theory,Littlestone1988learning,li2008knows,auer2008near}. One of the main measures is that of \emph{regret}, which is the difference between the agent's (expected) rewards compared to that of the optimal policy. For finite-horizon MDPs, the regret over $K$ episodes (or $T=KH$ time steps) is defined as
\begin{align*}
    \operatorname{Regret}_{H}(T) := \sum_{k=1}^K \big(V_1^\ast(s^{(k)}_1) - V_1^{\pi^{(k)}}(s^{(k)}_1) \big),
\end{align*}
where $\pi^{(k)}$ and $s_1^{(k)}$ are the policy and the environment's initial state in the $k$-th episode, respectively. 
For infinite-horizon discounted MDPs, the regret over $T$ time steps is
\begin{align*}
    \operatorname{Regret}_{\infty}^\gamma(T) := \sum_{t=1}^T \big(V_1^{\ast,\gamma}(s_t) - V_1^{\pi,\gamma}(s_t) \big),
\end{align*}
where $s_1$ is the initial state and $\pi$ the employed policy. Finally, for infinite-horizon undiscounted MDPs, the regret over $T$ time steps is\footnote{$Tg^\ast - \sum_{t=1}^T r(s_t,a_t)$ is also used \cite{auer2006logarithmic,auer2008near}.}
\begin{align*}
    \operatorname{Regret}_{\infty}(T) &:= Tg^\ast - \mathbb{E}_{s_1=s}^\pi\left[\sum_{t=1}^T r(s_t,a_t)\right]\\
    &= \sum_{t=1}^T \left(g^\ast - g^{\pi_t^\infty}(s_t) \right).
\end{align*}
%

In the classical setting, there are several proposed algorithms that achieve regret bounds sublinear in the number of time steps $o(T)$ for finite-horizon MDPs~\cite{auer2008near,osband2014model,azar2017minimax,jin2018qlearning,zanette2019tighter,efroni2019tight,jin2020provably,yang2020reinforcement}. To our knowledge, the best regret bounds are due Azar et al.~\cite{azar2017minimax} and Zanette and Brunskill~\cite{zanette2019tighter}, which match the lower bound $\Omega(\sqrt{HSAT})$ from~\cite{auer2008near} for a certain range of parameters. For infinite-horizon discounted MDPs, the scenario is similar. Several different algorithms have been developed~\cite{brafman2002r,strehl2006pac,Wang2020Qlearning,zhang2021model,liu2020regret,szita2011agnostic,lattimore2012pac}, with He et al.~\cite{he2021nearly} reaching a regret bound of $\widetilde{O}(\sqrt{\Gamma^3 SAT})$ and thus matching a lower bound from the same work for a certain range of parameters. For infinite-horizon undiscounted MDPs, the landscape is richer. The algorithm of Auer et al.~\cite{auer2008near} ranked among the main ones to be first proposed, although it applied to a smaller class of MDPs called \emph{communicating}~\cite{puterman2014markov}. Bartlett and Tewari~\cite{bartlett2009regal} achieved the regret bound $\widetilde{O}(\Lambda\sqrt{S^2AT})$ for the broader class of weakly communicating MDPs, where $\operatorname{sp}(h^\ast) \leq \Lambda$. Later, Refs.~\cite{ortner2012online,lakshmanan2015improved} adapted the algorithm from~\cite{auer2008near} for the case when the state space is $[0,1]^D$ under the assumption that $r$ and $p$ are H\"older continuous. The algorithms of~\cite{bartlett2009regal,ortner2012online,lakshmanan2015improved} are, however, time inefficient, which was later fixed by Fruit et al.~\cite{fruit2018efficient}, whose time-efficient algorithm still maintains the regret bound $\widetilde{O}(\Lambda\sqrt{S^2AT})$. 
We summarise several results from the literature in \Cref{table:results_finite-horizon}.

In the quantum setting, the number of works on online learning of MDPs is much smaller. We are only aware of \cite{zhong2023provably,ganguly2023quantum} for finite-horizon MDPs and~\cite{ganguly2025quantum} for infinite-horizon undiscounted MDPs (there are further works on multi-armed bandits~\cite{Lumbreras2022multiarmedquantum,dai2023quantum,wan2023quantum,su2025quantum}). By leveraging a modified environment-agent interaction model (described below), both works obtained $\poly(S,A,H,\log{T})$ regret, exponentially better in $T$ (see \Cref{table:results_finite-horizon}).


\subsection{Our RL model} 
\label{sec:our_rl_model}

Even though the interaction model between agent and environment is straightforward in the classical setting, the same is not true in the quantum one. Ideally, we would like to employ the quantum oracle $\mathcal{Q}_p$ from \eqref{eq:quantum_sampling_oracle} to explore the MDP in superposition, which inevitably leads to some apparent conundrums. For once, a repeated interaction between agent and environment through quantum oracles will lead to a superposition of different possible rewards. It is then not clear what the final regret is, especially if the agent performs non-trivial intermediary quantum gates and measurements. More critically, we would like to employ quantum subroutines that make use of the inverse of $\mathcal{Q}_p$~\cite{brassard2002quantum,montanaro2015quant,cornelissen2022near,Kothari2023mean}. There is absolutely no equivalent inverse operation in the standard classical model and, moreover, one would imagine that ``inverting'' quantum operations could potentially ``undo'' the accumulated regret.

We solve these issues by virtually separating the exploration phase from the policy learning phase. In the standard classical model, these are usually perform simultaneously: while interacting with the environment, the agent keeps track of all state-action pairs $(s_t,a_t)$ observed so far to come up with an estimation $\widetilde{p}$ of the true transition probability, which is then used to obtain an approximate optimal policy. In our model, however, such interaction is split into \emph{two} types of phases: \emph{classical online (exploration)} phases and \emph{classical/quantum offline (generative)} phases.
\begin{enumerate}[wide = 0pt, leftmargin = *]
    \item \textbf{Online phase:} corresponds exactly to the standard classical agent-environment interaction, during which the agent accumulates regret. More specifically, the agent chooses action $a_t$ at state $s_t$, obtains reward $r(s_t,a_t)$, and observes the new state $s_{t+1}$. The interaction is completely classical and lasts as long as the agent desires.
    \item \textbf{Offline phase:} the agent is \emph{free} to interact with the environment using the oracle $\mathcal{Q}_p$ \emph{without} accumulating regret. This means that the agent is free to prepare any quantum state and have the environment apply $\mathcal{Q}_p$ (or its inverse) onto such quantum state, plus any quantum gate from a universal gate set. The transition from online to offline phase and vice versa can be done at any moment. During the offline phase, the agent can \emph{only use the oracle $\mathcal{Q}_p$ and its inverse at most $O(\tau^\beta)$ times,} where $\tau$ is length of the previous \emph{online phase} and $\beta\in[0,\infty)$ is a fixed parameter called \emph{budget}.
\end{enumerate}

By separating the accumulation of regret from the policy learning phase as above, all the problems previously highlighted are avoided. The result is an alternation between online and offline phases. The offline phase can also be referred to as ``generative'' because the agent is free to interact with the environment in a generative fashion and, as described below, we employ the algorithms from \Cref{sec:computing_optimal_policies} to compute an $\varepsilon$-optimal policy during this phase. The restriction on the number of applications of $\mathcal{Q}_p$ during the offline phase is vital to guarantee a meaningful online-learning model, otherwise the agent could obtain a policy as close to optimal as needed and thus incur a regret as small as desired. In order to freely use the probability oracle $\mathcal{Q}_p$, the agent must first pay a price in committing to actions and incurring into regret. The level of freedom in offline phases or, equivalently, the level of commitment required in online phases, is regulated by the budget parameter $\beta\in[0,\infty)$. Larger budgets $\beta$ lead to more computational time, better policies, and consequently smaller regrets. Finally, we also consider a classical version of the above model wherein, during offline phases, the agent has access to the sampling oracle $\mathcal{C}_p$.

The above online-offline RL model clearly generalises the standard RL model wherein the length of the offline phase is zero. 
If the classical RL environment can be viewed as a computer program with a source code, then simulation-like access to $\mathcal{Q}_p$ can be obtained by translating the source code into a Boolean circuit that draws samples from the distribution $p$ and by employing standard techniques in quantum computing to convert such a Boolean circuit into a quantum circuit implementing $\mathcal{Q}_p$. Alternatively, from a more physical perspective, the environment could potentially be described by a quantum open system and access to $\mathcal{Q}_p$ be obtained via Markov chain techniques~\cite{Temme2025quantizedmarkov}. More generally, our online-offline RL model can describe the interaction between an agent and a (black-box) quantum computer in an online-fashion wherein the agent has limited access to the quantum computer due to, e.g., high traffic or monetary reasons, interspersed with periods of free interaction due to the quantum computer being idle.

\paragraph{A critique on previous quantum RL works.} Although the splitting of exploration and learning processes into two separate phases can be seen as powerful, it solves the fundamental problem that the regret is not properly defined when the agent interacts with the environment in a quantum fashion. Unfortunately, as far as we know, there is currently \emph{no meaningful and useful} quantum RL model that does not assume a regret-free offline phase. \emph{All} prior works with $\operatorname{poly}\log{T}$ regret had to assume, even if implicitly, a regret-free offline phase. This is due to a fundamental flaw in the algorithms of~\cite{zhong2023provably,ganguly2023quantum,ganguly2025quantum}, which we quickly review.

Refs.~\cite{zhong2023provably,ganguly2023quantum,ganguly2025quantum} proposed quantum RL algorithms with access to the oracle $\mathcal{Q}_p$ and its inverse, their main idea being that quantum multi-dimensional amplitude estimation~\cite{van2021quantum,hamoudi2021quantum} can be employed to learn the stochastic kernels $p(\cdot|s,a)$ quadratically faster than a classical counterpart. Their quantum algorithms collect quantum samples of the form $\mathcal{Q}_p|s_t\rangle,\mathcal{Q}_p^\dagger|s_t\rangle$ while exploring the structure of the MDP, see \cite[Algorithm~1, Line~8]{zhong2023provably}, \cite[Algorithm~1, Line~12]{ganguly2023quantum}, and \cite[Algorithm~2, Line~11]{ganguly2025quantum}.\footnote{The RL model of~\cite{ganguly2023quantum} is not as clear. They assume access to ``q-random variables'' (a concept from~\cite{hamoudi2021quantum} which includes access to oracle $\mathcal{Q}_p$) at the agent's end, but also assume ``that the quantum environment can generate multiple copies of the next state superposition and so all q-random variables and the next state measurement can be obtained'', which we understand as gathering $\mathcal{Q}_p|s_t\rangle,\mathcal{Q}_p^\dagger|s_t\rangle$.} Then, when enough quantum samples have been collected, the authors argue in \cite[Algorithm~1, Line~11]{zhong2023provably}, \cite[Algorithm~1, Line~6]{ganguly2023quantum}, and \cite[Algorithm~2, Line~21]{ganguly2025quantum} that quantum amplitude estimation is implemented using the stored states $\{\mathcal{Q}_p |s_{t}\rangle,\mathcal{O}^{\dagger}_p |s_{t}\rangle\}_t$. However, this is \emph{incorrect} because quantum amplitude estimation requires sequential and coherent applications of $\mathcal{Q}_p$ and $\mathcal{Q}_p^\dagger$ onto a \emph{single} quantum state, and not access to a collection of states $\{\mathcal{Q}_p|s_t\rangle,\mathcal{Q}_p^\dagger|s_t\rangle\}_t$. The results of~\cite{zhong2023provably,ganguly2023quantum,ganguly2025quantum} are thus wrong. Nonetheless, it is possible to fix these issues by employing a RL model similar to ours, with groups of exploration steps $s_t \to s_{t+1}$ forming an online phase and one use of quantum mean estimation forming an offline phase. Every time \cite[Algorithm~1]{zhong2023provably}, \cite[Algorithm~1]{ganguly2023quantum}, and \cite[Algorithm~2]{ganguly2025quantum} use quantum amplitude estimation~\cite{van2021quantum,hamoudi2021quantum}, they are implicitly entering an offline phase.

Under this correction, the corresponding online phases of~\cite{zhong2023provably,ganguly2023quantum,ganguly2025quantum} are virtually the same as ours. Although~\cite{zhong2023provably} phrased the online interaction in a quantum fashion by using an oracle for policies, there is absolutely no need for such and the agent-environment interaction can be done entirely classically. The RL model of~\cite{zhong2023provably} does have an extra restriction on the offline phase compared to~\cite{ganguly2023quantum,ganguly2025quantum} and ours, in that the environment's state $|s_{t}\rangle$ is fixed from the end of the previous online phase and applications of $\mathcal{Q}_p,\mathcal{Q}_p^\dagger$ are done on such state $|s_{t}\rangle$ only. Here we assume that the agent can manipulate the environment's state during offline phases and apply $\mathcal{Q}_p,\mathcal{Q}_p^\dagger$ onto states that the agent prepares. Finally, the length of offline phases in~\cite{zhong2023provably,ganguly2023quantum,ganguly2025quantum} is linearly proportional to the previous online phase as implicit when they use quantum amplitude estimation.

\subsection{Our regret bounds}

Using our online-offline RL model described above, we propose new classical and quantum algorithms that achieve better regret bounds compared to prior works. Our algorithms for finite- and infinite-horizon MDPs are similar and are explained in \Cref{algo:intro}. They progress in phases, each made up of an online phase and an offline phase. During the $\ell$-th offline phase, an approximately optimal policy $\pi^{(\ell)}$ is computed using the algorithms from \Cref{sec:computing_optimal_policies}. Such policy is then employed in the online phase to accumulate regret for a period of time long enough to allow the computation of a better policy in the next offline phase. If the $\ell$-th online phase lasts $\tau_{\ell}$ time steps, then in the $\ell$-th offline phase one can make $O(\tau_{\ell-1}^\beta)$ calls to the classical or quantum oracles $\mathcal{C}_p$ or $\mathcal{Q}_p$, which, according to \Cref{sec:computing_optimal_policies}, yields an $\varepsilon_{\ell}$-optimal policy with $\varepsilon_\ell = \widetilde{O}(\tau_{\ell-1}^{-\beta/2})$ in the classical setting and $\varepsilon_\ell = \widetilde{O}(\tau_{\ell-1}^{-\beta})$ in the quantum setting (the dependence on other parameters is shown in \Cref{algo:intro}). In order to set the length of the online phases, which determine the number of oracle calls in the offline phases, we follow a doubling trick: the length of the online phase is increased by a constant factor after each episode. For finite-horizon MDPs, the increase factor is $2H$, while for infinite-horizon MDPs it is $2$.


Although very similar, the main difference between the finite- and infinite-horizon versions of \Cref{algo:intro} is that each online phase in the finite-horizon setting must be a multiple of $H$, since the agent explores the MDP in sequences of $H$ steps. This is implicit in the definition of $\operatorname{Regret}_H(T)$, where $V_1^\ast(s_1^{(\ell)}) - V_1^{\pi^{(\ell)}}(s_1^{(\ell)})$ is a difference of rewards over $H$ steps. As a consequence, the agent computes a policy $\pi^{(\ell)}=(\pi_1^{(\ell)},\pi_2^{(\ell)},\dots,\pi_H^{(\ell)})$ for $H$ time steps during the offline phase (hence why we write $\pi_{t~(\operatorname{mod}H)}^{(\ell)}$ in \Cref{algo:intro} to compactly represent the repeated application over all decision rules in $\pi^{(\ell)}$). In the infinite-horizon setting, the agent computes a stationary policy $\pi_\ell^\infty = (\pi_\ell,\pi_\ell,\dots)$ during the offline phase, meaning that the same decision rule $\pi_\ell$ is employed throughout the next online phase.

%
\begin{algorithm*}[t!]
    \caption{Online-learning algorithms for MDPs}
    \label{algo:intro}
    \begin{algorithmic}[1]
    \Require Finite state space $\mathcal{S}$ and action space $\mathcal{A}$, rewards $r:\mathcal{S}\times\mathcal{A}\to [0,1]$, horizon $H$ (finite-horizon), discount factor $\gamma\in[0,1)$ (infinite-horizon discounted), upper bound $\Lambda$ on optimal bias span (infinite-horizon undiscounted).
    
    \State $t \gets 1$ and $\tau_1 \gets 1$
    
    \For{phase $\ell=1, 2, \dots$}
        
        \phase{Generative phase}

        \State \textbf{Finite-horizon:} Using \Cref{fact:finite_horizon_classical} or \cref{res:res1}, choose policy $\pi^{(\ell)}$ such that w.h.p.\
        \begin{align*}
            \|V_1^{\pi^{(\ell)}} - V_1^\ast\|_\infty \leq  \begin{cases}
                \widetilde{O}\Big(\sqrt{\frac{H^2 SA}{\tau_\ell}} \Big) &\text{(classical)},\\
                \widetilde{O}\Big({\min}\Big\{\frac{H^{1.5}SA}{\tau_\ell},\frac{H^2S\sqrt{A}}{\tau_\ell}\Big\} \Big) &\text{(quantum)}
            \end{cases}
        \end{align*}

        \setcounter{ALG@line}{2}
        \State \textbf{Infinite-horizon discounted:} Using \Cref{fact:discounted_classical} or \Cref{fact:discounted_quantum}, choose decision rule $\pi_\ell$ s.t.\ w.h.p.\
        \begin{align*}
            \|V_1^{\pi_\ell^\infty,\gamma} - V_1^{\ast,\gamma}\|_\infty \leq \begin{cases}
                \widetilde{O}\Big(\sqrt{\frac{\Gamma^3 SA}{\tau_\ell}}\Big) &\text{(classical)},\\
                \widetilde{O}\Big({\min}\Big\{\frac{\Gamma^{1.5}SA}{\tau_\ell},\frac{\Gamma^3 S\sqrt{A}}{\tau_\ell}\Big\} \Big) &\text{(quantum)}
            \end{cases} 
        \end{align*}

        \setcounter{ALG@line}{2}
        \State \textbf{Infinite-horizon undiscounted:} Using \Cref{fact:undiscounted_classical} or \cref{res:res2}, choose decision rule $\pi_\ell$ s.t.\ w.h.p.\
        \begin{align*}
            \|g^{\pi_\ell^\infty} - g^\ast \mathbf{1}\|_\infty \leq \begin{cases}
                \widetilde{O}\Big(\sqrt{\frac{\Lambda SA}{\tau_\ell}}\Big) &\text{(classical)},\\
                \widetilde{O}\Big(\frac{\Lambda S\sqrt{A}}{\tau_\ell}\Big) &\text{(quantum)}
            \end{cases} 
        \end{align*}

        \phase{Exploration phase}
        
        \State $\tau_{\ell+1} \gets t$ and observe random initial state $s_{t}$
        \While {$t < 2H\tau_{\ell+1}$ (finite-horizon) / $t < 2\tau_{\ell+1}$ (infinite-horizon)}
            \State Choose action $a_{t}\sim \pi^{(\ell)}_{t~(\operatorname{mod} H)}(s_{t})$ (finite-horizon) or $a_{t}\sim \pi_\ell(s_{t})$ (infinite-horizon)
            \State Obtain reward $r_{t}\gets r(s_{t},a_{t})$, observe next state $s_{t+1}\sim p(\cdot|s_{t},a_{t})$, and $t \gets t + 1$
        \EndWhile
        
    \EndFor
\end{algorithmic}
\end{algorithm*}

\paragraph*{Finite-horizon MDPs}
\begin{result}\label{res:regret_finite_horizon}
    Let $\langle \mathcal{S},\mathcal{A},p,r,H\rangle$ be a finite-horizon MDP. Under the RL model of {\rm \Cref{sec:our_rl_model}} with budget $\beta$, there is a classical algorithm with regret $\operatorname{Regret}_H(T)$ upper-bounded after $T$ steps with probability $1-\delta$ by (up to $\poly\log\log$ terms)
    \begin{align*}
        \begin{cases}
            \widetilde{O}\left(\sqrt{HSAT^{2-\beta}\log\!\frac{SAT}{\delta}\log{T}} \right) &\text{if}~\beta < 2, \\
            \widetilde{O}\left(\sqrt{HSA\log\!\frac{T}{H}\log\!\frac{SAT}{\delta}\log{T}} \right) &\text{if}~\beta = 2, \\
            \widetilde{O}\left(\sqrt{H^{3-\beta}SA\log\!\frac{HSA}{\delta}\log{H}} \right) &\text{if}~\beta > 2.
        \end{cases}
    \end{align*}
    Similarly, there is a quantum algorithm with regret $\operatorname{Regret}_H(T)$ upper-bounded after $T$ steps with probability $1-\delta$ by (up to $\poly\log\log$ terms)
    \begin{align*}
        \resizebox{\hsize}{!}{$\begin{cases}
            \widetilde{O}\!\left(\min\{H^{1.5}SA,H^2S\sqrt{A}\}T^{1-\beta}\log^2\!\frac{SAT}{\delta} \right) &\text{if}~\beta < 1, \\
            \widetilde{O}\!\left(\min\{H^{1.5}SA,H^2S\sqrt{A}\}\log\!\frac{T}{H}\log^2\!\frac{SAT}{\delta} \right) &\text{if}~\beta = 1, \\
            \widetilde{O}\!\left(\min\{H^{2.5-\beta}SA,H^{3-\beta}S\sqrt{A}\}\log^2\!\frac{HSA}{\delta} \right) &\text{if}~\beta > 1.
        \end{cases}$}
    \end{align*}
\end{result}
\begin{proof}
    We start with the classical setting. The total regret over $T = HK$ time steps is
    \begin{align*}
        \operatorname{Regret}_{H}(T) := \sum_{k=1}^K\big( V_1^\ast(s_1^{(k)}) - V_1^{\pi^{(k)}}(s_1^{(k)})\big).
    \end{align*}
    The policy $\pi^{(\ell)}$ is updated every time $t = H 2^\ell$, to a total of $\lceil \log_2{K}\rceil$ times, so that $\ell\in\{1,\dots,\lceil\log_2{K}\rceil\}$ marks the $\ell$-th phase. The policies $\{\pi^{(k)}\}_{k\in[K]}$ from the regret equation above are defined as $\pi^{(k)} = \pi^{(\ell)}$ for $2^{\ell - 1} \leq k < 2^\ell$. At the beginning of phase $\ell$, when the length of the previous online phase is $\tau_{\ell-1} = H 2^{\ell-1}$, we employ \Cref{fact:finite_horizon_classical} with $O(\tau_{\ell-1}^\beta) = O((H2^\ell)^\beta)$ calls to oracle $\mathcal{C}_p$ in order to obtain a policy $\pi^{(\ell)}$ such that, with probability at least $1-\frac{\delta}{\lceil\log_2{K}\rceil}$, for all $s\in\mathcal{S}$,
    \begin{align*}
        &V_1^\ast(s) - B_c\sqrt{\frac{\ln\!\big(\frac{HSA2^\ell}{\delta}\big)\ln(H2^\ell)}{2^{\ell\beta}}} \leq V_1^{\pi^{(\ell)}}(s), \\
        &\text{where}~ B_c = \widetilde{O}\big(\sqrt{H^{3-\beta}SA}\big)
    \end{align*}
    and $\widetilde{O}(\cdot)$ omits $\poly\log\log$ terms. Hence, with probability at least $1-\delta$ over all policy updates, $\operatorname{Regret}_{H}(T)$ is upper-bounded by (up to constant factors)
    \begin{align*}
        &B_c\sum_{k=1}^K \sqrt{\frac{\ln\!\big(\frac{kHSA}{\delta}\big)\ln(kH)}{k^{\beta}}} = \\
        &\begin{cases}
            O\Big(B_c K^{1-\beta/2}\sqrt{\log\!\big(\frac{KHSA}{\delta}\big)\log(KH)}\Big) &\text{if}~\beta < 2, \\
            O\Big( B_c \log(K)\sqrt{\log\!\big(\frac{KHSA}{\delta}\big)\log(KH)}\Big) &\text{if}~\beta = 2,\\
            O\Big(B_c \sqrt{\log\!\big(\frac{HSA}{\delta}\big)\log{H}}\Big) &\text{if}~\beta > 2,
        \end{cases} 
    \end{align*}
    using that $\sum_{k=1}^K \frac{1}{k} = O(\log{K})$, $\sum_{k=1}^K \frac{1}{k^{x}} = O(K^{1-x})$ for $x < 1$, and that $\sum_{k=1}^K \frac{\ln{k}}{k^{x}} < \sum_{k=1}^\infty \frac{\ln{k}}{k^{x}}$ is a constant for $x>1$. The result follows since $KH = T$.

    The quantum setting is very similar. At the beginning of phase $\ell$, when the length of the previous online phase is $\tau_{\ell-1} = H 2^{\ell-1}$, we employ \Cref{res:res1} with $O(\tau_{\ell-1}^\beta) = O((H2^\ell)^\beta)$ calls to oracle $\mathcal{Q}_p$ in order to obtain a policy $\pi^{(\ell)}$ such that, with probability at least $1-\frac{\delta}{\lceil\log_2{K}\rceil}$, for all $s\in\mathcal{S}$,
    \begin{align*}
        &V_1^\ast(s) - B_q\frac{\ln^2\!\big(\frac{HSA2^\ell}{\delta}\big)}{2^{\ell\beta}} \leq V_1^{\pi^{(\ell)}}(s) , \\
        &\text{where}~ B_q = \widetilde{O}\big({\min}\{H^{2.5-\beta}SA,H^{3-\beta}S\sqrt{A}\}\big)
    \end{align*}
    and $\widetilde{O}(\cdot)$ omits $\poly\log\log$ terms. Hence, with probability at least $1-\delta$ over all policy updates, $\operatorname{Regret}_{H}(T)$ is upper-bounded by (up to constant factors)
    \begin{align*}
        &B_q\sum_{k=1}^K \frac{\ln^2\!\big(\frac{kHSA}{\delta}\big)}{k^{\beta}} = \\
        &\begin{cases}
            O\Big(B_q K^{1-\beta}\log^2\!\big(\frac{KHSA}{\delta}\big)\Big) &\text{if}~\beta < 1, \\
            O\Big( B_q \log(K)\log^2\!\big(\frac{KHSA}{\delta}\big)\Big) &\text{if}~\beta = 1,\\
            O\Big(B_q \log^2\!\big(\frac{HSA}{\delta}\big)\Big) &\text{if}~\beta > 1,
        \end{cases} 
    \end{align*}
    using similar series inequalities as in the classical setting. The result follows by taking $T = HK$.
\end{proof}

\paragraph*{Infinite-horizon discounted MDPs}

\begin{result}\label{res:regret_discounted}
    Let $\langle \mathcal{S},\mathcal{A},p,r,\gamma\rangle$ be an infinite-horizon discounted MDP. Under the RL model of {\rm \Cref{sec:our_rl_model}} with budget $\beta$, there is a classical algorithm with regret $\operatorname{Regret}_{\infty}^\gamma(T)$ upper-bounded after $T$ steps with probability $1-\delta$ by (up to $\poly\log\log$ terms)
    \begin{align*}
        \begin{cases}
            \widetilde{O}\left(\sqrt{\Gamma^3 SAT^{2-\beta}\log\!\frac{\Gamma SAT}{\delta}\log(\Gamma T)} \right) &\text{if}~\beta < 2, \\
            \widetilde{O}\left(\sqrt{\Gamma^3 SA\log{T} \log\!\frac{\Gamma SAT}{\delta}\log(\Gamma T)}\right) &\text{if}~\beta = 2, \\
            \widetilde{O}\left(\sqrt{\Gamma^{3}SA\log\!\frac{\Gamma SA}{\delta}\log{\Gamma}} \right) &\text{if}~\beta > 2.
        \end{cases}
    \end{align*}
    Similarly, there is a quantum algorithm with regret $\operatorname{Regret}_{\infty}^\gamma(T)$ upper-bounded after $T$ steps with probability $1-\delta$ by (up to $\poly\log\log$ terms)
    \begin{align*}
        \resizebox{\hsize}{!}{$\begin{cases}
            \widetilde{O}\!\left(\min\{\Gamma^{1.5}SA,\Gamma^3 S\sqrt{A}\}T^{1-\beta}\log^4\!\frac{\Gamma SAT}{\delta} \right) &\text{if}~\beta < 1, \\
            \widetilde{O}\!\left(\min\{\Gamma^{1.5} SA,\Gamma^3 S\sqrt{A}\}\log{T}\log^4\!\frac{\Gamma SAT}{\delta} \right) &\text{if}~\beta = 1, \\
            \widetilde{O}\!\left(\min\{\Gamma^{1.5}SA,\Gamma^{3}S\sqrt{A}\}\log^4\!\frac{\Gamma SA}{\delta} \right) &\text{if}~\beta > 1.
        \end{cases}$}
    \end{align*}
\end{result}
\begin{proof}
    We start with the classical setting. The total regret over $T$ time steps is
    \begin{align*}
        \operatorname{Regret}_{\infty}^\gamma(T) := \sum_{t=1}^T \big( V_1^{\ast,\gamma}(s_t) - V_1^{\pi_t^\infty,\gamma}(s_t)\big),
    \end{align*}
    where $\pi_t$ is the decision rule used at step $t$. The policy $\pi_t^\infty$ is updated every time $t = 2^\ell$, to a total of $\lceil \log_2{T}\rceil$ times, so that $\ell\in\{1,\dots,\lceil\log_2{T}\rceil\}$ marks the $\ell$-th phase. The decision rules $\{\pi_t\}_{t}$ from the regret equation above are thus defined as $\pi_{t} = \pi_{2^\ell}$ for $2^{\ell} \leq t < 2^{\ell+1}$. At the beginning of phase $\ell$, when the length of the previous online phase is $\tau_{\ell-1} = 2^{\ell-1} = t/2$, we employ \Cref{fact:discounted_classical} with $O(\tau_{\ell-1}^\beta) = O(t^\beta)$ calls to oracle $\mathcal{C}_p$ in order to obtain a policy $\pi^{\infty}_t$ such that, with probability at least $1-\frac{\delta}{\lceil\log_2{T}\rceil}$, for all $s\in\mathcal{S}$,
    \begin{align*}
        &V_1^{\ast,\gamma}(s) - B_c\sqrt{\frac{\ln\!\big(\frac{t\Gamma SA}{\delta}\big)\ln(t\Gamma)}{t^\beta}} \leq V_1^{\pi^\infty_t,\gamma}(s), \\
        &\text{where}~ B_c = \widetilde{O}\big(\sqrt{\Gamma^{3}SA}\big)
    \end{align*}
    and $\widetilde{O}(\cdot)$ omits $\poly\log\log$ terms. Hence, with probability at least $1-\delta$ over all policy updates, $\operatorname{Regret}_{\infty}^\gamma(T)$ is upper-bounded by (up to constant factors)
    \begin{align*}
        &B_c\sum_{t=1}^T \sqrt{\frac{\ln\!\big(\frac{t\Gamma SA}{\delta}\big)\ln(t\Gamma)}{t^{\beta}}}  = \\
        &\begin{cases}
            O\Big(B_c T^{1-\beta/2}\sqrt{\log\!\big(\frac{\Gamma SAT}{\delta}\big)\log(\Gamma T)}\Big) &\text{if}~\beta < 2, \\
            O\Big( B_c \log(T)\sqrt{\log\!\big(\frac{\Gamma SAT}{\delta}\big)\log(\Gamma T)}\Big) &\text{if}~\beta = 2,\\
            O\Big(B_c \sqrt{\log\!\big(\frac{SA}{\delta}\big)\log{\Gamma}}\Big) &\text{if}~\beta > 2,
        \end{cases} 
    \end{align*}
    using that $\sum_{t=1}^T \frac{1}{t} = O(\log{T})$, $\sum_{t=1}^T \frac{1}{t^{x}} = O(T^{1-x})$ for $x < 1$, and that $\sum_{t=1}^T \frac{\ln{t}}{t^{x}}$ is a constant for $x>1$.

    The quantum setting is very similar. At the beginning of phase $\ell$, when the length of the previous online phase is $\tau_{\ell-1} = 2^{\ell-1} = t/2$, we employ \Cref{fact:discounted_quantum} with $O(\tau_{\ell-1}^\beta) = O(t^\beta)$ calls to oracle $\mathcal{Q}_p$ in order to obtain a policy $\pi_t^\infty$ such that, with probability at least $1-\frac{\delta}{\lceil\log_2{T}\rceil}$, for all $s\in\mathcal{S}$,
    \begin{align*}
        &V_1^{\ast,\gamma}(s) - B_q\frac{\ln^4\!\big(\frac{t\Gamma SA}{\delta}\big)}{t^{\beta}} \leq V_1^{\pi_t^\infty,\gamma}(s), \\
        &\text{where}~ B_q = \widetilde{O}\big({\min}\{\Gamma^{1.5}SA,\Gamma^{3}S\sqrt{A}\}\big)
    \end{align*}
    and $\widetilde{O}(\cdot)$ omits $\poly\log\log$ terms. Hence, with probability at least $1-\delta$ over all policy updates, $\operatorname{Regret}_{\infty}^\gamma(T)$ is upper-bounded by (up to constant factors)
    \begin{align*}
        &B_q\sum_{t=1}^T \frac{\ln^4\!\big(\frac{t\Gamma SA}{\delta}\big)}{t^{\beta}} = \\
        &\begin{cases}
            O\big(B_q T^{1-\beta}\log^4\!\big(\frac{\Gamma SAT}{\delta}\big)\big) &\text{if}~\beta < 1, \\
            O\big( B_q \log(T)\log^4\!\big(\frac{\Gamma SAT}{\delta}\big)\big) &\text{if}~\beta = 1,\\
            O\big(B_q \log^4\!\big(\frac{\Gamma SA}{\delta}\big)\big) &\text{if}~\beta > 1,
        \end{cases} 
    \end{align*}
    using the series inequalities as in the classical case.
\end{proof}

\paragraph*{Infinite-horizon undiscounted MDPs}

\begin{result}\label{res:regret_undiscounted}
    Let $\langle \mathcal{S},\mathcal{A},p,r\rangle$ be a infinite-horizon undiscounted weakly communicating MDP with $\operatorname{sp}(h^\ast) \leq \Lambda$. Under the RL model of {\rm \Cref{sec:our_rl_model}} with budget $\beta$, there is a classical algorithm with regret $\operatorname{Regret}_\infty(T)$ upper-bounded after $T$ steps with probability $1-\delta$ by (up to $\poly\log\log$ terms)
    \begin{align*}
        \begin{cases}
            \widetilde{O}\left(\sqrt{\Lambda SA T^{2-\beta} \log\!\frac{SAT}{\delta}}\right) &\text{if}~\beta < 2,\\
            \widetilde{O}\left(\sqrt{\Lambda SA \log^2{T} \log\!\frac{SAT}{\delta}} \right) &\text{if}~\beta = 2,\\
            \widetilde{O}\left(\sqrt{\Lambda SA \log\!\frac{SAT}{\delta}}\right) &\text{if}~\beta > 2.
        \end{cases}
    \end{align*}
    Similarly, under the extra assumption that the optimal Bellman operator is a $J$-stage $\nu$-span contraction for constant $J$ and $\nu$, there is a quantum algorithm with regret $\operatorname{Regret}_\infty(T)$ upper-bounded after $T$ steps with probability $1-\delta$ by (up to $\poly\log\log$ terms)
    \begin{align*}
        \begin{cases}
            \widetilde{O}\left( \Lambda S\sqrt{A} T^{1-\beta}\log{T}\log^2\!\frac{SA}{\delta} \right) &\text{if}~\beta < 1,\\
            \widetilde{O}\left(\Lambda S\sqrt{A} \log^2{T}\log^2\!\frac{SA}{\delta} \right) &\text{if}~\beta = 1,\\
            \widetilde{O}\left(\Lambda S\sqrt{A} \log^2\!\frac{SA}{\delta} \right) &\text{if}~\beta > 1.
        \end{cases}
    \end{align*}
\end{result}
\begin{proof}
    We start with the classical setting. The total expected regret after $T$ time steps is
    \begin{align*}
        \operatorname{Regret}_{\infty}(T) &= \sum_{t=1}^T \left(g^\ast - g^{\pi_t^\infty}(s_t) \right),
    \end{align*}
    where $\pi_t = \pi_{2^\ell}$ is the deterministic decision rule employed in steps $2^{\ell} \leq t < 2^{\ell+1}$ during phase $\ell$. The stationary policy $\pi_\ell^\infty$ is updated every time $t=2^\ell$, to a total of $\lceil\log_2{T}\rceil$ times. At the beginning of phase $\ell$, when the previous online phase is $\tau_{\ell-1} = 2^{\ell-1} = t/2$, we employ \Cref{fact:undiscounted_classical} with $O(\tau_{\ell-1}^\beta) = O(t^{\beta/2})$ calls to oracle $\mathcal{C}_p$ to obtain a new policy $\pi_\ell^\infty$ such that, with probability at least $1 - \frac{\delta}{\lceil\log_2{T}\rceil}$, for all $s\in\mathcal{S}$,
    \begin{align*}
        g^{\pi_\ell^\infty}(s) &\geq g^\ast - B_c\sqrt{\frac{\ln\!\big(\frac{SAt}{\delta}\big)}{t^{\beta}}}, 
        ~\text{where}~ B_c = \widetilde{O}(\sqrt{\Lambda S A})
    \end{align*}
    and $\widetilde{O}(\cdot)$ omits $\poly\log\log$ terms. Hence, with probability at least $1-\delta$ over all policy updates, $\operatorname{Regret}_{\infty}(T)$ is upper-bounded by (up to constant factors)
    \begin{align*}
         &B_c\sum_{t=1}^T \sqrt{\frac{\ln\!\big(\frac{SAt}{\delta}\big)}{t^\beta}} = \\
         &\begin{cases}
             O\Big(B_c T^{1-\beta/2} \sqrt{\log\!\big(\frac{SAT}{\delta}\big)}\Big) &\text{if}~\beta < 2,\\
             O\Big(B_c \log{T} \sqrt{\log\!\big(\frac{SAT}{\delta}\big)}\Big) &\text{if}~\beta = 2,\\
             O\Big(B_c \sqrt{\log\!\big(\frac{SA}{\delta}\big)}\Big) &\text{if}~\beta > 2. 
         \end{cases}
    \end{align*}
    
    Regarding the quantum setting, at the beginning of phase $\ell$, when the previous online phase is $\tau_{\ell-1} = 2^{\ell-1} = t/2$, we employ \Cref{res:res2} with $O(\tau_{\ell-1}^\beta) = O(t^\beta)$ calls to oracle $\mathcal{Q}_p$ to obtain a new policy $\pi_\ell^\infty$ such that, with probability at least $1 - \frac{\delta}{\lceil\log_2{T}\rceil}$, for all $s\in\mathcal{S}$,
    \begin{align*}
        g^{\pi_\ell^\infty}(s) &\geq g^\ast - B_q\frac{\ln{t}}{t^\beta}, \\
        &\text{where}~ B_q = \widetilde{O}\left(\Lambda S\sqrt{A}\log^2\left(\frac{SA}{\delta}\right)\right)
    \end{align*}
    and $\widetilde{O}(\cdot)$ omits $\poly\log\log$ terms. Hence, with probability at least $1-\delta$ over all policy updates, $\operatorname{Regret}_{\infty}(T)$ is upper-bounded by (up to constant factors)
    \begin{align*}
         B_q\sum_{t=1}^T \frac{\ln{t}}{t^\beta} = \begin{cases}
             O(B_q T^{1-\beta}\log{T}) &\text{if}~\beta < 1,\\
             O(B_q \log^2{T}) &\text{if}~\beta = 1,\\
             O(B_q) &\text{if}~\beta > 1. 
         \end{cases}\tag*{\qedhere}
    \end{align*}
\end{proof}

\begin{table*}[t]
\centering
\caption{Summary of several known upper bounds for $\operatorname{Regret}_H(T)$ of \emph{finite-horizon} MDPs $\langle \mathcal{S},\mathcal{A},p,r,H\rangle$, $\operatorname{Regret}_{\infty}^\gamma(T)$ of \emph{infinite-horizon discounted} MDPs $\langle \mathcal{S},\mathcal{A},p,r,\gamma\rangle$, and $\operatorname{Regret}_\infty(T)$ of \emph{infinite-horizon undiscounted} MDPs $\langle \mathcal{S},\mathcal{A},p,r\rangle$ with state space $\mathcal{S}$ of size $S$, action space $\mathcal{A}$ of size $A$, number of time steps $T$, horizon $H$, effective horizon $\Gamma = (1-\gamma)^{-1}$, and $\operatorname{sp}(h^\ast) \leq \Lambda$. Here $D \geq \Lambda$ is the diameter and $\tau_{\rm mix} \geq \Lambda$ the mixing time of the MDP. All bounds are up to $\poly\log$ factors. Refs.\ marked by $\dagger$ have no efficient implementation for infinite-horizon undiscounted MDPs. The reader should keep in mind that not all works assume the same RL model: there are model-based~\cite{auer2008near} and model-free~\cite{jin2018qlearning} approaches, while~\cite{zhong2023provably,ganguly2023quantum,ganguly2025quantum} and our work assume a hybrid online-offline model (here with budget parameter $\beta=1$ for simplicity).}
\def\arraystretch{1.4}
\resizebox{\linewidth}{!}{
\begin{tabular}{|c|cc|cc|cc|}
\hline

\multirow{2}{*}{Works}  & \multicolumn{2}{c|}{Finite-Horizon} & \multicolumn{2}{c|}{Infinite-Horizon Discounted} & \multicolumn{2}{c|}{Infinite-Horizon Undiscounted} \\ 

& \multicolumn{1}{c|}{Classical} & Quantum & \multicolumn{1}{c|}{Classical} & Quantum & \multicolumn{1}{c|}{Classical} & Quantum \\ \hline 

\cite{auer2008near} & \multicolumn{1}{c|}{$\sqrt{H^2S^2AT} + HSA$} & - & \multicolumn{1}{c|}{-} & - & \multicolumn{1}{c|}{$\sqrt{D^2S^2AT} + D SA$} & - \\ \hline

\cite{bartlett2009regal}$^\dagger$ & \multicolumn{1}{c|}{$\sqrt{H^2S^2AT}$} & - & \multicolumn{1}{c|}{-} & - & \multicolumn{1}{c|}{$\sqrt{\Lambda^2S^2AT}$} & - \\ \hline

\cite{ouyang2017learning} & \multicolumn{1}{c|}{-} & - & \multicolumn{1}{c|}{-} & - & \multicolumn{1}{c|}{$\sqrt{\Lambda^2 S^2 AT}$} & - \\ \hline

\cite{azar2017minimax} & \multicolumn{1}{c|}{$\sqrt{HSAT} + H^2S^2A + \sqrt{H^2T}$} & - & \multicolumn{1}{c|}{-} & - & \multicolumn{1}{c|}{-} & - \\ \hline


\cite{fruit2018efficient} & \multicolumn{1}{c|}{-} & - & \multicolumn{1}{c|}{-} & - & \multicolumn{1}{c|}{$\sqrt{\Lambda^2 S^2 AT} + \Lambda S^2A$} & - \\ \hline

\multirow{2}{*}{\cite{jin2018qlearning}} & \multicolumn{1}{c|}{$\sqrt{H^4SAT}$} & - & \multicolumn{1}{c|}{-} & - & \multicolumn{1}{c|}{-} & - \\ 

& \multicolumn{1}{c|}{$\sqrt{H^3SAT} + \sqrt{H^9S^3A^3}$} & - & \multicolumn{1}{c|}{-} & - & \multicolumn{1}{c|}{-} & - \\ \hline

\cite{zanette2019tighter,efroni2019tight} & \multicolumn{1}{c|}{$\!\!\sqrt{HSAT} \!+\! H^2 S^{\frac{3}{2}} A(\sqrt{S} \!+\! \sqrt{H})\!\!$} & - & \multicolumn{1}{c|}{-} & - & \multicolumn{1}{c|}{-} & - \\ \hline

\cite{bai2019provably} & \multicolumn{1}{c|}{$\sqrt{H^3SAT} + \sqrt{H^9S^3A^3}$} & - & \multicolumn{1}{c|}{-} & - & \multicolumn{1}{c|}{-} & - \\ \hline

\cite{zhang2019regret}$^\dagger$ & \multicolumn{1}{c|}{-} & - & \multicolumn{1}{c|}{-} & - & \multicolumn{1}{c|}{$\!\!\sqrt{\Lambda S AT} \!+\! \Lambda(S^{10} A T)^{\frac{1}{4}}\!\!$} & - \\ \hline

\cite{zhang2020almost} & \multicolumn{1}{c|}{$\sqrt{H^2SAT} + H^8 S^2 A^{\frac{3}{2}} T^{\frac{1}{4}}\!$} & - & \multicolumn{1}{c|}{-} & - & \multicolumn{1}{c|}{-} & - \\ \hline

\cite{fruit2020improved} & \multicolumn{1}{c|}{-} & - & \multicolumn{1}{c|}{-} & - & \multicolumn{1}{c|}{$\sqrt{D S^2 AT}$} & - \\ \hline

\cite{ortner2020regret}$^\dagger$ & \multicolumn{1}{c|}{-} & - & \multicolumn{1}{c|}{-} & - & \multicolumn{1}{c|}{$\sqrt{\tau_{\rm mix} S AT}$} & - \\ \hline

\cite{wei2020model} & \multicolumn{1}{c|}{-} & - & \multicolumn{1}{c|}{-} & - & \multicolumn{1}{c|}{$\Lambda (S AT^2)^{\frac{1}{3}}$} & - \\ \hline

\cite{liu2020regret} & \multicolumn{1}{c|}{-} & - & \multicolumn{1}{c|}{$\sqrt{\Gamma^5SAT} + \Gamma^2SA$} & - & \multicolumn{1}{c|}{-} & - \\ \hline

\cite{menard2021ucb} & \multicolumn{1}{c|}{$\sqrt{H^2SAT} + H^4 S A$} & - & \multicolumn{1}{c|}{-} & - & \multicolumn{1}{c|}{-} & - \\ \hline

\cite{li2021breaking} & \multicolumn{1}{c|}{$\sqrt{H^2SAT} + H^6 S A$} & - & \multicolumn{1}{c|}{-} & - & \multicolumn{1}{c|}{-} & - \\ \hline

\multirow{2}{*}{\cite{wei2021learning}} & \multicolumn{1}{c|}{-} & - & \multicolumn{1}{c|}{-} & - & \multicolumn{1}{c|}{$\sqrt{\Lambda^2 S^3 A^3 T}$} & - \\

& \multicolumn{1}{c|}{-} & - & \multicolumn{1}{c|}{-} & - & \multicolumn{1}{c|}{$\!\!\!\sqrt{\Lambda} (S A T)^{\frac{3}{4}} \!+\! (\Lambda S A T)^{\frac{2}{3}}\!\!\!$} & - \\ \hline

\cite{he2021nearly} & \multicolumn{1}{c|}{-} & - & \multicolumn{1}{c|}{$\!\Gamma^{\frac{7}{2}}S^2A^{\frac{3}{2}} \!+\! \sqrt{\Gamma^3SAT} \!+\! \sqrt{\Gamma^4T}\!$} & - & \multicolumn{1}{c|}{-} & - \\ \hline

\cite{zhang2023sharper} & \multicolumn{1}{c|}{-} & - & \multicolumn{1}{c|}{-} & - & \multicolumn{1}{c|}{$\sqrt{\Lambda^2 S^{10} A^4 T}$} & - \\ \hline

\cite{zhong2023provably} & \multicolumn{1}{c|}{-} & $H^3S^2A$ & \multicolumn{1}{c|}{-} & - & \multicolumn{1}{c|}{-} & - \\ \hline

\cite{ganguly2023quantum} & \multicolumn{1}{c|}{-} & $H^2S^2A$ & \multicolumn{1}{c|}{-} & - & \multicolumn{1}{c|}{-} & - \\ \hline

\cite{ganguly2025quantum} & \multicolumn{1}{c|}{-} & & \multicolumn{1}{c|}{-} & - & \multicolumn{1}{c|}{-} & $\Lambda S^5 A^4$ \\ \hline

\cite{agrawal2024optimistic} & \multicolumn{1}{c|}{$\sqrt{H^{12}S^2 AT} + H^9 S^2 A$} & - & \multicolumn{1}{c|}{-} & - & \multicolumn{1}{c|}{-} & - \\ \hline

\rowcolor{Gray}
\!\!This work\!\! & \multicolumn{1}{c|}{$\sqrt{HSAT}$} & $\!\!\min\{H^{\frac{3}{2}}SA,H^2S\sqrt{A}\}\!\!$ & \multicolumn{1}{c|}{$\sqrt{\Gamma^3 SAT}$} & $\!\!\min\{\Gamma^{\frac{3}{2}} SA, \Gamma^3S\sqrt{A}\}\!\!$ & \multicolumn{1}{c|}{$\sqrt{\Lambda SAT}$} & $\Lambda S \sqrt{A}$ \\ \hline

\end{tabular}}
\label{table:results_finite-horizon}
\end{table*}

Several regret bounds from \Cref{res:regret_finite_horizon,res:regret_discounted,res:regret_undiscounted} considerably improve upon prior works~\cite{auer2008near,azar2017minimax,zanette2019tighter,liu2020regret,he2021nearly,zhong2023provably,ganguly2023quantum} when the budget parameter is $\beta \geq 1$ (see \Cref{table:results_finite-horizon} for a clear comparison)\footnote{It is clear that prior classical regret bounds under the standard RL model (with no offline phase) always apply and thus subsume the bounds from \Cref{res:regret_finite_horizon,res:regret_discounted,res:regret_undiscounted} when $\beta$ is too small.}. Taking $\beta=1$ for simplicity, for finite-horizon MDPs, our classical bound $\widetilde{O}(\sqrt{HSAT})$ matches the ones from~\cite{azar2017minimax,zanette2019tighter} when $T\geq H^3S^3A$ and $SA\geq H$, and avoids the extra terms $\widetilde{O}(H^2S^2A + H\sqrt{T})$ outside this parameter range. The quantum bound $\widetilde{O}(\min\{H^{1.5} SA,H^2S\sqrt{A}\})$ substantially improves the dependence on $S$, $A$, $H$ compared to~\cite{zhong2023provably,ganguly2023quantum} and maintains the logarithmic dependence on $T$. Regarding infinite-horizon discounted MDPs, our classical bound $\widetilde{O}(\sqrt{\Gamma^3 SAT})$ matches the ones from~\cite{he2021nearly} when $T = \Omega(\Gamma^4 S^3A^2)$ while avoiding the extra terms $\widetilde{O}(\Gamma^{3.5} S^2A^{1.5} + \sqrt{\Gamma^4 T})$. The quantum bound $\widetilde{O}(\min\{\Gamma^{1.5} SA,\Gamma^3 S\sqrt{A}\})$, on the other hand, is brand new and has only a polylogarithmic dependence on $T$. Finally, for undiscounted MDPs, our classical regret $\widetilde{O}(\sqrt{\Lambda SAT})$ improves upon the bound $\widetilde{O}(\Lambda\sqrt{S^2AT})$ of~\cite{bartlett2009regal,fruit2018efficient}, while our quantum regret bound $\widetilde{O}(\Lambda S\sqrt{A})$ greatly improves on the bound $\widetilde{O}(\Lambda S^5 A^4)$ of~\cite{ganguly2025quantum}. However, one must keep in mind that such comparisons are made between different RL models.

The key to the improved performance of \Cref{algo:intro} is to avoid standard RL principles like “optimism in the face of uncertainty”~\cite{brafman2002r} common in several RL algorithms~\cite{auer2008near,bartlett2009regal,ortner2012online,lakshmanan2015improved,azar2017minimax,jin2018qlearning,zanette2019tighter,efroni2019tight,zhong2023provably,ganguly2023quantum}. In these algorithms, each state-action pair is given some ``optimism'' such that its imagined value is as high as statistically possible and the agent chooses a policy according to such optimistic values. In other words, the estimated transition probability $\widetilde{p}$ together with its uncertainty define a set $\mathcal{M}$ of plausible MDPs that contains the true MDP with high probability. The agent then chooses an optimal policy with respect to \emph{all} MDPs in $\mathcal{M}$ (optimism). Our algorithms, on the other hand, avoid this principle altogether since, during the offline phases, we have access to the true MDP via oracles $\mathcal{C}_p$ or $\mathcal{Q}_p$. Ultimately, an optimal policy is what is needed to maintain a small regret. Therefore, there is no reason to approximate $p$ and the agent instead directly computes an approximate optimal policy using the algorithms from  \Cref{sec:computing_optimal_policies}. The agent can eventually learn $p$ later on when a good policy has already been obtained.

\section{Conclusions and future work}

We formalised an online-offline RL model wherein the agent can freely interact with the environment from time to time via specific oracles encoding the underlying probability structure of an MDP. Combined with known classical and novel quantum algorithms for computing $\varepsilon$-optimal policies, we proposed brand new classical and quantum online algorithms for learning MDPs with better regret bounds compared to prior works. This includes the proposal of a novel measure of regret with respect to which our quantum RL algorithm has exponentially better regret in $T$ compared to its classical counterpart. Our results show the impact of having a bit of ``freedom'' between agent and environment and stimulates the exploration and comparison of new RL models.

The measure of freedom in our RL model is capture by the budget parameter $\beta\in[0,\infty)$. Under very little freedom during the offline phases ($\beta < 1$), both classical and quantum algorithms have $\poly(T)$ regret. However, when there is enough but not overwhelming freedom ($\beta\in[1,2)$), our quantum bounds improve exponentially the dependence on $T$ to $\poly\log(T)$ while their classical counterparts still present regrets as large as $O(\sqrt{T})$. Finally, under a lot of freedom ($\beta\geq 2$), the exponential quantum advantage in $T$ vanishes as the classical regret bounds now depend on $\poly\log{T}$. We have thus identified a range of budget parameter $\beta\in[1,2)$ for which quantum computers can bring a substantial advantage to the learning process.

A main open question is to formalise a meaningful quantum RL model without offline phases which can still achieve $\operatorname{poly}\log{T}$ regret. As previous mentioned, no such model is currently known and all previous works had to assume the existence of such an offline phase. Another question is proving regret lower bounds under our RL model.

\subsubsection*{Acknowledgments}
DL thanks Patrick Rebentrost for helpful discussions. JFD thanks Alessandro Luongo and Miklos Santha for comments on the manuscript, and Yecheng Xue for clarifications on Ref.~\cite{zhong2023provably}. AA was supported by Latvian Quantum Initiative under European Union Recovery and Resilience Facility project No.\ 2.3.1.1.i.0/1/22/I/CFLA/001. DL acknowledges funding from QuantERA Project QOPT. JFD is supported by ERC grant No.\ 810115-DYNASNET. Part of this project was done while JFD was a visiting researcher at the Simons Institute for the Theory of Computing.

\bibliography{iclr2026_conference}
\bibliographystyle{unsrt}

\newpage
\appendix
\onecolumn

\section{Quantum algorithms for optimal policies}
\label{sec:optimal_policies_finite-horizon}

In this section, we prove our results concerning the computation of approximate optimal policies given access to oracle $\mathcal{Q}_p$ and its inverse. In the following, given bounded real-valued functions $u,v\in\mathscr{B}(\mathcal{S})$, let $(uv)(s) = u(s)v(s)$ for all $s\in\mathcal{S}$. Moreover, given a randomised decision rule $d$, define
\begin{align*}
    r_{d}(s) := \operatorname*{\mathbb{E}}_{a\sim d(\cdot|s)}[r(s,a)] \qquad\text{and} \qquad p_{d}(s'|s) := \operatorname*{\mathbb{E}}_{a\sim d(\cdot|s)}[p(s'|s,a)].
\end{align*}
Given $u\in\mathscr{B}(\mathcal{S})$ and a randomised decision rule $d$, define $P_d u\in\mathscr{B}(\mathcal{S})$ as $(P_d u)(s) := \sum_{s'\in\mathcal{S}} p_d(s'|s) u(s')$. We shall abuse notation and let $(P_a u)(s) := \sum_{s'\in\mathcal{S}} p(s'|s,a) u(s')$ for an action $a\in\mathcal{A}$.

We shall require the following quantum subroutines.
\begin{fact}[Quantum max-finding \cite{durr1996quantum}]\label{fact:quantum_minimum_finding}
    Given quantum access to $u\in\mathscr{B}(\mathcal{S})$ via oracle $\mathcal{O}_u$, one can find $\max_{s\in\mathcal{S}} u(s)$ and $\min_{s\in\mathcal{S}} u(s)$ with probability $1-\delta$ using $O(\sqrt{S}\log \frac{1}{\delta})$ queries to $\mathcal{O}_u$.
\end{fact}

\begin{fact}[Quantum mean estimation with variance {\cite[Theorem~1.1]{Kothari2023mean}}]\label{fact:quantum_mean_estimation_variance}
    Let $\epsilon>0$ and $\delta\in (0, 1)$. Assume quantum access to function $u:\mathcal{S}\to \mathbb{R}$ via oracle $\mathcal{O}_u$ and quantum sampling access to probability distribution $p\in\Delta(\mathcal{S})$ via oracle $\mathcal{O}_p$. Let $\sigma := \sum_{s\in\mathcal{S}} p(s) u(s)^2 - \big(\sum_{s\in\mathcal{S}} p(s) u(s)\big)^2$. There is a quantum algorithm that computes $\widetilde \mu\in\mathbb{R}$ such that $| \widetilde\mu - \sum_{s\in\mathcal{S}} p(s) u(s)|\leq \sqrt{\sigma}\epsilon$ with success probability $1-\delta$ using $O\big(\frac{1}{\epsilon}\log\frac{1}{\delta}\big)$ queries to $\mathcal{O}_u,\mathcal{O}_p$, and their inverses.
\end{fact}

\begin{corollary}[Quantum mean estimation]\label{lem:quantum_mean_estimation}
    Let $\epsilon>0$ and $\delta\in (0, 1)$. Assume quantum access to function $u:\mathcal{S}\to \mathbb{R}_+$ with known $\min_{s\in\mathcal{S}}u(s)$ and $\max_{s\in\mathcal{S}}u(s)$ via oracle $\mathcal{O}_u$ and quantum sampling access to probability distribution $p\in\Delta(\mathcal{S})$ via oracle $\mathcal{O}_p$. There is a quantum algorithm that computes $\widetilde \mu \in\mathbb{R}$ such that $| \widetilde\mu - \sum_{s\in\mathcal{S}} p(s) u(s)|\leq \operatorname{sp}(u)\epsilon$ with success probability $1-\delta$ using $O\big(\frac{1}{\epsilon}\log\frac{1}{\delta}\big)$ queries to $\mathcal{O}_u,\mathcal{O}_p$, and their inverses.
\end{corollary}

\subsection{Finite-Horizon MDPs}

We start by considering finite-horizon MDPs $\langle \mathcal{S},\mathcal{A},p,r,H\rangle$. The basis of our algorithms is the simple backward induction algorithm used to find the so-called \emph{value functions}
\begin{align*}
    V_t^\pi(s) := \mathbb{E}^\pi_{s_{t}=s}\left[\sum_{t'=t}^{H} r(s_{t'},a_{t'}) \right] \qquad\forall s\in\mathcal{S}, t\in[H].
\end{align*}
For $t=H,H-1,\dots,1$, compute
\begin{align*}
    u_{t}(s) &= (\mathcal{L} u_{t+1})(s) = \max_{a\in\mathcal{A}}\{r_a(s) + (P_a u_{t+1})(s)\} = \max_{a\in\mathcal{A}}\left\{r(s,a) + \sum_{s'\in\mathcal{S}} p(s'|s,a)u_{t+1}(s')\right\}, \\
    \pi_{t}(s) &\in \argmax_{a\in\mathcal{A}}\{(\mathcal{L}_a u_{t+1})(s)\} = \argmax_{a\in\mathcal{A}}\{r_a(s) + (P_a u_{t+1})(s)\}.
\end{align*}
Let the deterministic policy $\pi = (\pi_t)_{t\in[H]}$. One can prove that $u_{t}(s) = V_t^{\pi}(s) = V_t^\ast(s)$ for all $s\in\mathcal{S}$. This means that $u_{t}$ is the optimal total expected reward from time $t$ onward and $\pi$ is an optimal policy~\cite[Theorem~4.5.1]{puterman2014markov}. 

Our first quantum algorithm (\Cref{algo:quantum_backward_recursion}) is based on the modern backward induction algorithm from Sidford et al.~\cite{sidford2018near} and relies on the monotonicity, variance reduction, and total-variance techniques introduced in~\cite{sidford2018near} and briefly explained in \Cref{sec:computing_optimal_policies}. We shall require the following technical upper bound on the accumulated variance.

\begin{fact}[{\cite[Lemma~F.4]{sidford2018near}}]\label{fact:upper_bound_variance2}
    Given a policy $\pi$, let $\sigma^\pi_{t}\in\mathscr{B}(\mathcal{S})$ be defined as $\sigma_{t}^\pi = P_{\pi_t} (V_{t+1}^\pi)^2 - (P_{\pi_t} V_{t+1}^\pi)^2$ for $t\in[H-1]$. For any policy $\pi$ and $t\in[H-1]$,\footnote{We note the typo in \cite[Lemma~F.4]{sidford2018near} where $\|\cdot\|_\infty^2$ should be $\|\cdot\|_\infty$.}
    \begin{align*}
        \left\|\sum_{t'=t}^{H-1} \left(\prod_{i=t}^{t'-1} P_{\pi_i} \right) \sqrt{\sigma_{t'}^\pi} \right\|_{\infty} \leq H^{3/2}.
    \end{align*}
\end{fact}

\begin{algorithm}[t!]
    \caption{Modern quantum backward induction algorithm}
    \label{algo:quantum_backward_recursion}
    \begin{algorithmic}[1]  
    \Require Finite state space $\mathcal{S}$ and action space $\mathcal{A}$, horizon $H$, quantum sampling access to probability kernels $p$, failure probability $\delta\in(0, 1)$, error $\varepsilon > 0$.

    \Ensure $\varepsilon$-optimal deterministic policy $\pi^{(K)}$. 

    \State Initialise $u^{(0)}_t \equiv 0$ for all $t\in[H]$ and arbitrary deterministic policy $\pi^{(0)}$ 

    \For{$k=1$ to $K:=\lceil\log_2(H/\varepsilon)\rceil$}

        \State $\epsilon_k \gets H/2^k$ and $\theta_k \gets \frac{\min\{\epsilon_k,1\}}{13 H^{3/2}}$
    
    
        \For{$(s,a,t)\in\mathcal{S}\times\mathcal{A}\times[H]$}

            \State \parbox[t]{\dimexpr\linewidth-\algorithmicindent-\algorithmicindent}{Obtain $\widetilde{\sigma}^{(k)}_t(s,a)$ s.t.\ $\big|\widetilde{\sigma}^{(k)}_t(s,a) - \sigma^{(k)}_t(s,a)\big| \leq \theta_kH^2$ via \Cref{fact:quantum_mean_estimation_variance}, where $\sigma^{(k)}_t(s,a) = \sum_{s'\in\mathcal{S}}p(s'|s,a)u_t^{(k-1)}(s')^2 - \big(\sum_{s'\in\mathcal{S}}p(s'|s,a)u_t^{(k-1)}(s')\big)^2$} \label{line:quantum_backward_line1}
    
            \State \parbox[t]{\dimexpr\linewidth-\algorithmicindent-\algorithmicindent}{Obtain $\widetilde{\mu}^{(k)}_t(s,a)$ s.t.\ $\big| \widetilde{\mu}^{(k)}_t(s,a) - \mu_t^{(k)}(s,a)\big| \leq \theta_k \sqrt{\sigma_{t}^{(k)}(s,a)}$ via \Cref{fact:quantum_mean_estimation_variance}, where $\mu_t^{(k)}(s,a) = \sum_{s'\in\mathcal{S}}p(s'|s,a)u_t^{(k-1)}(s')$} \label{line:quantum_backward_line2}
    
        \EndFor

        \State $\widehat{\mu}_t^{(k)}(s,a) \gets \widetilde{\mu}_t^{(k)}(s,a) - \theta_k\sqrt{\widetilde{\sigma}_{t}^{(k)}(s,a)} - \theta_k^{3/2}H$ for all $(s,a,t)\in\mathcal{S}\times\mathcal{A}\times[H]$

        \State $u^{(k)}_H(s) \gets u^{(k-1)}_{H}(s)$ and $\pi_H^{(k)}(s)\gets \pi_{H}^{(k-1)}(s)$ for all $s\in\mathcal{S}$

        \For{$t=H-1,H-2,\dots,1$}

            \For{$(s,a)\in\mathcal{S}\times\mathcal{A}$}

            \State \parbox[t]{\dimexpr\linewidth-\algorithmicindent-\algorithmicindent-\algorithmicindent}{Obtain $\widetilde{\beta}_{t+1}^{(k)}(s,a)$ s.t.\ $\big|\widetilde{\beta}_{t+1}^{(k)}(s,a) - \beta^{(k)}_{t+1}(s,a) \big| \leq \frac{\epsilon_k}{4H}$ via \Cref{lem:quantum_mean_estimation}, where $\beta^{(k)}_{t+1}(s,a) = \sum_{s'\in\mathcal{S}}p(s'|s,a)\big(u^{(k)}_{t+1}(s') -  u^{(k-1)}_{t+1}(s')\big)$} \label{line:quantum_backward_line3}

            \State $\widehat{\beta}^{(k)}_{t+1}(s,a) \gets \widetilde{\beta}_{t+1}^{(k)}(s,a) - \frac{\epsilon_k}{4H}$

            \EndFor

            \For{$s\in\mathcal{S}$}

            \State $u^{(k)}_t(s) \gets \max_{a\in\mathcal{A}}\{r(s,a) + \widehat{\mu}_{t+1}^{(k)}(s,a) + \widehat{\beta}_{t+1}^{(k)}(s,a)\}$
            
            \State $\pi_t^{(k)}(s) \gets \argmax_{a\in\mathcal{A}}\{r(s,a) + \widehat{\mu}_{t+1}^{(k)}(s,a) + \widehat{\beta}_{t+1}^{(k)}(s,a)\}$

            \State If $u^{(k)}_t(s) \leq u^{(k-1)}_{t}(s)$, then $u^{(k)}_t(s) \gets u^{(k-1)}_{t}(s)$ and $\pi_t^{(k)}(s) \gets \pi_{t}^{(k-1)}(s)$
            
            \EndFor

        \EndFor

    \EndFor
    
    \State {\bfseries return} $\pi^{(K)}$
\end{algorithmic}
\end{algorithm}

\begin{theorem}\label{thr:quantum_finite-horizon}
    Let $M = \langle \mathcal{S},\mathcal{A}, p, r,H\rangle$ be a finite-horizon MDP. Let $K := \lceil\log_2(H/\varepsilon)\rceil$ for a given $\varepsilon > 0$ and $\epsilon_k := H/2^k$ for all $k\in[K]$. {\rm \Cref{algo:quantum_backward_recursion}} computes functions $\{u^{(k)}_t\}_{t\in[H],k\in[K]}\subset\mathscr{B}(\mathcal{S})$ and deterministic policies $\{\pi^{(k)}\}_{k\in[K]}$ such that, with probability $1-\delta$,
    \begin{align*}
        V^\ast_t(s) - \epsilon_k \leq u^{(k)}_t(s) \leq V^{\pi^{(k)}}_t(s) \leq  V^\ast_t(s) \qquad\forall (s,t,k)\in\mathcal{S}\times[H]\times[K].
    \end{align*}
    In particular, $\pi^{(K)}$ is such that $V^\ast_1(s) - \varepsilon \leq V^{\pi^{(K)}}_1(s) \leq  V^\ast_1(s)$ $\forall s\in\mathcal{S}$. The $\mathcal{Q}_p/\mathcal{Q}_p^\dagger$-query complexity is 
    \begin{align*}
        O\left(\frac{H^{5/2}SA}{\varepsilon}\log\left(\frac{HSA}{\delta}\log\frac{H}{\varepsilon} \right) \right).
    \end{align*}
\end{theorem}
\begin{proof}
    The proof is by induction on the epoch $k=0,1,\dots,K$, the base case $k=0$ being trivial. Assume then that \Cref{algo:quantum_backward_recursion} has properly computed functions and policies such that
    \begin{align}\label{eq:quantum_backward_eq0}
        V^\ast_t(s) - \epsilon_{k'} \leq u^{(k')}_{t}(s) \leq (\mathcal{L}_{\pi^{(k')}_{t}}u^{(k')}_{t+1})(s) \leq  V^\ast_t(s) \quad \forall k'\in[k-1], t\in[H],
    \end{align}
    and consider epoch $k\in[K]$. Let $\theta_k := \frac{1}{20 H^{3/2}}\min\{\epsilon_k, 1\}$. We start the analysis with the quantities $\widetilde{\mu}_t^{(k)}$ and $\widetilde{\sigma}_t^{(k)}$ from \Cref{line:quantum_backward_line1,line:quantum_backward_line2}. Define the true quantities
    \begin{align*}
        \mu^{(k)}_t(s,a) &:= \sum_{s'\in\mathcal{S}}p(s'|s,a)u_{t}^{(k-1)}(s') = (P_a u_t^{(k-1)})(s), \\
        \sigma_t^{(k)}(s,a) &:= \sum_{s'\in\mathcal{S}}p(s'|s,a)u_{t}^{(k-1)}(s')^2 - \left(\sum_{s'\in\mathcal{S}}p(s'|s,a)u_{t}^{(k-1)}(s')\right)^2 = (P_a (u_t^{(k-1)})^2)(s) - (P_a u_t^{(k-1)})^2(s).
    \end{align*}
    For all $(s,a)\in\mathcal{S}\times\mathcal{A}$, we employ $H$ different times the quantum mean estimation from \Cref{fact:quantum_mean_estimation_variance} with
    \begin{align*}
        m_k := O\left(\frac{1}{\theta_k}\log\frac{HSAK}{\delta}\right) = O\left(\frac{H^{3/2}}{\min\{\epsilon_k,1\}}\log\frac{HSAK}{\delta}\right)
    \end{align*}
    queries to $\mathcal{Q}_p$ to obtain $\widetilde{\sigma}_1^{(k)}(s,a), \dots, \widetilde{\sigma}_H^{(k)}(s,a)$ (each quantity $\sigma_t^{(k)}(s,a)$ has standard deviation at most $H^{2}$). Likewise, for all $(s,a)\in\mathcal{S}\times\mathcal{A}$, we employ $H$ different times the quantum mean estimation subroutine from \Cref{fact:quantum_mean_estimation_variance} with
    \begin{align*}
        n_k := O\left(\frac{1}{\theta_k}\log\frac{HSAK}{\delta}\right) = O\left(\frac{H^{3/2}}{\min\{\epsilon_k,1\}}\log\frac{HSAK}{\delta}\right)
    \end{align*}
    queries $\mathcal{Q}_p$ to obtain $\widetilde{\mu}_1^{(k)}(s,a),\dots,\widetilde{\mu}_H^{(k)}(s,a)$. These quantities are such that, with probability $1 - \frac{\delta}{2K}$, for all $(s,a,t)\in\mathcal{S}\times\mathcal{A}\times[H]$,
    \begin{subequations}
    \begin{align}
        |\widetilde{\mu}^{(k)}_t(s,a) - \mu^{(k)}_t(s,a)| &\leq \theta_k\sqrt{\sigma_{t}^{(k)}(s,a)}, \label{eq:quantum_inequality_1}\\
        |\widetilde{\sigma}_t^{(k)}(s,a) - \sigma_t^{(k)}(s,a)| &\leq \theta_k H^2, \label{eq:quantum_inequality_2}
    \end{align}
    \end{subequations}
    which we condition on. By using \eqref{eq:quantum_inequality_2} onto \eqref{eq:quantum_inequality_1}, then
    \begin{align*}
        |\widetilde{\mu}^{(k)}_t(s,a) - \mu^{(k)}_t(s,a)| \leq \theta_k \sqrt{\widetilde{\sigma}_{t}^{(k)}(s,a)} + \theta_k^{3/2} H.
    \end{align*}
    from which we define $\widehat{\mu}_t^{(k)}(s,a)$, for all $(s,a)\in\mathcal{S}\times\mathcal{A}$, as
    \begin{align*}
        \widehat{\mu}^{(k)}_t(s,a) := \widetilde{\mu}^{(k)}_t(s,a) - \theta_k\sqrt{\widetilde{\sigma}_{t}^{(k)}(s,a)} - \theta_k^{3/2} H,
    \end{align*}
    which has one-sided error. We can express the above quantity using the variance $\sigma_t^\ast(s,a) := (P_a (V_t^\ast)^2)(s) - (P_a V_t^\ast)^2(s)$, since, for all $(s,a)\in\mathcal{S}\times\mathcal{A}$,
    \begin{align*}
        \sqrt{\widetilde{\sigma}_t^{(k)}(s,a)} \leq \sqrt{\sigma_t^{(k)}(s,a)} + \sqrt{\theta_k}H&\leq \sqrt{\sigma^\ast_t(s,a)} + 2\epsilon_k + \sqrt{\theta_k}H,
    \end{align*}
    using \eqref{eq:quantum_inequality_2}, $\sqrt{\operatorname{Var}[u^{(k-1)}_{t}]} \leq \sqrt{\operatorname{Var}[V^\ast_t]} + \sqrt{\operatorname{Var}[V^\ast_t - u^{(k-1)}_{t}]}$ and $\operatorname{Var}[V^\ast_t - u^{(k-1)}_{t}] \leq (2\epsilon_{k})^2$ if $\|V^\ast_t - u^{(k-1)}_{t}\|_\infty \leq 2\epsilon_{k}$ according to the induction hypothesis. This means that, for all $(s,a,t)\in\mathcal{S}\times\mathcal{A}\times[H]$, with probability at least $1-\frac{\delta}{2K}$,
    \begin{subequations}\label{eq:quantum_backward_eq1}
    \begin{align}
        \widehat{\mu}^{(k)}_t(s,a) &\leq \mu_t^{(k)}(s,a), \label{eq:quantum_backward_eq1a} \\
        \widehat{\mu}^{(k)}_t(s,a) &\geq \mu_t^{(k)}(s,a) - 2\theta_k\sqrt{\sigma_{t}^\ast(s,a)} - 4\theta_k\epsilon_{k} - 4\theta_k^{3/2} H. \label{eq:quantum_backward_eq1b}
    \end{align}
    \end{subequations}

    We now proceed to the backward iteration that happens within epoch $k\in[K]$. To prove that $u_t^{(k)}(s) \leq V_t^{\pi^{(k)}}(s) \leq V^\ast_t(s)$, assume by induction that
    \begin{align*}
        u_{t'}^{(k)}(s) \leq (\mathcal{L}_{\pi^{(k)}_{t'}}u^{(k)}_{t'+1})(s) \leq V_{t'}^{\pi^{(k)}}(s) \leq V^\ast_{t'}(s) \qquad\forall s\in\mathcal{S}, t'=t+1,t+2,\dots,H+1.
    \end{align*}
    Once $u_{t'+1}^{(k)}\in\mathscr{B}(\mathcal{S})$ has been computed for all $t'=t,t+1,\dots, H-1$, we can compute the quantities $\widetilde{\beta}_{t+1}^{(k)}$ from \Cref{line:quantum_backward_line3} at time step $t\in[H]$. For all $(s,a)\in\mathcal{S}\times\mathcal{A}$, we employ the quantum mean estimation subroutine from \Cref{lem:quantum_mean_estimation} with $\ell_k := O\big(H\log\frac{HSAK}{\delta}\big)$ queries to $\mathcal{Q}_p$ to obtain $\widetilde{\beta}_{t+1}^{(k)}(s,a)$ such that, with probability $1-\frac{\delta}{2HK}$,
    \begin{align*}
       \left| \widetilde{\beta}_{t+1}^{(k)}(s,a) - \sum_{s'\in\mathcal{S}}p(s'|s,a)\big(u^{(k)}_{t+1}(s') - u^{(k-1)}_{t+1}(s')\big) \right| \leq \frac{\epsilon_k}{4H} \qquad \forall (s,a)\in\mathcal{S}\times\mathcal{A}
    \end{align*}
    (using that $\|u^{(k)}_{t+1} - u^{(k-1)}_{t+1}\|_\infty \leq 2\epsilon_k$ due to $u^{(k-1)}_{t+1}(s) \leq u^{(k)}_{t+1}(s)$ and the induction hypothesis both in $k$ and $t$).
    We then define
    \begin{align*}
        \widehat{\beta}^{(k)}_{t+1}(s,a) := \widetilde{\beta}_{t+1}^{(k)}(s,a) - \frac{\epsilon_{k}}{4H} \qquad \forall (s,a)\in\mathcal{S}\times\mathcal{A},
    \end{align*}
    from which, with probability at least $1-\frac{\delta}{2KH}$, for all $(s,a)\in\mathcal{S}\times\mathcal{A}$, it holds that
    \begin{subequations}\label{eq:quantum_backward_eq2}
    \begin{align}
        \widehat{\beta}^{(k)}_{t+1}(s,a) &\leq \sum_{s'\in\mathcal{S}}  p(s'|s,a)\big(u^{(k)}_{t+1}(s') - u^{(k-1)}_{t+1}(s') \big), \label{eq:quantum_backward_eq2a}\\
        \widehat{\beta}^{(k)}_{t+1}(s,a) &\geq \sum_{s'\in\mathcal{S}}  p(s'|s,a)\big(u^{(k)}_{t+1}(s') - u^{(k-1)}_{t+1}(s') \big) - \frac{\epsilon_{k}}{2H}. \label{eq:quantum_backward_eq2b}
    \end{align}
    \end{subequations}
    Conditioned on the above, there are two cases to consider. \textbf{Case 1: $\pi^{(k)}_t \neq \pi^{(k-1)}_{t}$.} If $\pi^{(k)}_t(s) \neq \pi^{(k-1)}_{t}(s)$ for some $s\in\mathcal{S}$, then it means that
    \begin{align*}
        u_t^{(k)}(s) &= r(s,\pi^{(k)}_t(s)) + \widehat{\mu}_{t+1}^{(k)}(s,\pi^{(k)}_t(s)) + \widehat{\beta}_{t+1}^{(k)}(s,\pi^{(k)}_t(s)) \\
        &\leq r(x,\pi^{(k)}_t(s)) + \sum_{s'\in\mathcal{S}}p(s'|s,\pi^{(k)}_t(s)) u^{(k)}_{t+1}(s') \tag{by \eqref{eq:quantum_backward_eq1a} and \eqref{eq:quantum_backward_eq2a}}\\
        &= (\mathcal{L}_{\pi^{(k)}_t}u^{(k)}_{t+1})(s).
    \end{align*}
    \textbf{Case 2: $\pi^{(k)}_t = \pi^{(k-1)}_{t}$.} If both policies are equal, then, for all $s\in\mathcal{S}$,
    \begin{align*}
        u^{(k)}_t(s) = u^{(k-1)}_{t}(s) \leq (\mathcal{L}_{\pi^{(k-1)}_{t}}u^{(k-1)}_{t+1})(s) \leq (\mathcal{L}_{\pi^{(k-1)}_{t}}u^{(k)}_{t+1})(s) = (\mathcal{L}_{\pi^{(k)}_{t}}u^{(k)}_{t+1})(s),
    \end{align*}
    where we used the induction hypothesis to argue that $u^{(k-1)}_{t}(s) \leq (\mathcal{L}_{\pi^{(k-1)}_{t}}u^{(k-1)}_{t+1})(s)$. From the above, we can readily see that $u^{(k)}_t(s) \leq V_t^\ast(s)$, which completes the induction on $t$.
    
    We now move on to proving that $u^{(k)}_t(s) \geq V^\ast_t(s) - \epsilon_k$ (assume that all $u^{(k)}_{t'+1}$ for $t' = t,\dots, H-1$ have already been computed by the backward iteration). Note that
    \begin{align*}
        V^\ast_t(s) - u^{(k)}_t(s) &= \max_{a\in\mathcal{A}}\left\{r(s,a) + \sum_{s'\in\mathcal{S}}p(s'|s,a)V_{t+1}^\ast(s') \right\} - \max_{a\in\mathcal{A}}\left\{r(s,a) + \widehat{\mu}_{t+1}^{(k)}(s,a) + \widehat{\beta}_{t+1}^{(k)}(s,a) \right\} \\
        &\leq \sum_{s'\in\mathcal{S}} p(s'|s,\pi^{\ast}_{t}(s))\big(V_{t+1}^\ast(s') - u^{(k)}_{t+1}(s')\big) + \xi_{t}^{(k)}(s), \tag{by \eqref{eq:quantum_backward_eq1b} and \eqref{eq:quantum_backward_eq2b}}
    \end{align*}
    where we defined
    \begin{align*}
        \xi_{t}^{(k)}(s) := 2\theta_k\sqrt{\sigma^\ast_{t+1}(s,\pi^\ast_{t}(s))} + \frac{\epsilon_{k}}{2H} + 4\theta_k\epsilon_{k} + 4\theta_k^{3/2}H .
    \end{align*}
    Solving the above recursion with the boundary condition $V^\ast_{H} \equiv u^{(k)}_{H}$, then (the inequality is entry-wise)
    \begin{align*}
        V^\ast_t - u^{(k)}_t \leq \sum_{t'=t}^{H-1}\left(\prod_{i=t}^{t'-1} P_{\pi^\ast_i}\right) \xi^{(k)}_{t'}.
    \end{align*}
    Using that (where $\sigma_{t}^\ast = P_{\pi_t^\ast} (V_{t+1}^\ast)^2 - (P_{\pi^\ast_t} V_{t+1}^\ast)^2\in\mathscr{B}(\mathcal{S})$)
    \begin{align*}
        \left\|\sum_{t'=t}^{H-1} \left(\prod_{i=t}^{t'-1}P_{\pi_i^\ast}\right) \sqrt{\sigma_{t'}^\ast}\right\|_\infty \leq H^{3/2}, \tag{by  \Cref{fact:upper_bound_variance2}}
    \end{align*}
    we finally get that
    \begin{align*}
        V^\ast_t(s) - u^{(k)}_t(s) &\leq 2\theta_k H^{3/2} + \frac{\epsilon_{k}}{2} + 4H\theta_k\epsilon_k + 4\theta_k^{3/2}H^2   \\
        &\leq \epsilon_k\left(\frac{2}{13} + \frac{1}{2} + \frac{4}{13 H^{1/2}} + \frac{4}{13^{3/2}H^{1/4}} \right) \tag{$\theta_k = \frac{1}{13 H^{3/2}}\min\{\epsilon_k, 1\}$}\\
        &\leq \epsilon_k.
    \end{align*}
    This concludes the proof of \eqref{eq:quantum_backward_eq0} for epoch $k\in[K]$ and thus for all epochs by induction.

    Regarding the failure probability, \eqref{eq:quantum_backward_eq1} holds with probability $1-\frac{\delta}{2K}$ for a single epoch, while \eqref{eq:quantum_backward_eq2} holds with probability $1-\frac{\delta}{2HK}$ for a single epoch and time step. Therefore, across all epochs and time steps, the success probability is at least $1-\delta$, as required. Finally, the total number of samples used is
    \[
        O\bigg(\sum_{k=1}^K H(m_k + n_k + \ell_k)SA\bigg) = O\left(\frac{H^{5/2}SA}{\varepsilon}\log\left(\frac{HSA}{\delta}\log\frac{H}{\varepsilon}\right) \right). \qedhere
    \]
\end{proof}

\begin{algorithm}[t!]
    \caption{Simple quantum backward induction algorithm}
    \label{algo:quantum_backward_recursion2}
    \begin{algorithmic}[1]  
    \Require Finite state space $\mathcal{S}$ and action space $\mathcal{A}$, horizon $H$, quantum sampling access to probability kernels $p$, failure probability $\delta\in(0, 1)$, error $\varepsilon > 0$.

    \Ensure $\varepsilon$-optimal deterministic policy $\pi$.

    \State $\widehat{u}_H(s) \gets \max_{a\in\mathcal{A}}\{r(s,a)\}$ and $\pi_H(s) \gets \argmax_{a\in\mathcal{A}}\{r(s,a)\}$ with probability $1-\frac{\delta}{HS}$ $\forall s\in\mathcal{S}$ (\Cref{fact:quantum_minimum_finding})
        
        \For{$t=H-1,H-2,\dots,1$}

            \For{$s\in\mathcal{S}$}

            \State \parbox[t]{\dimexpr\linewidth-\algorithmicindent-\algorithmicindent}{Use \Cref{lem:quantum_mean_estimation} to obtain a unitary $\mathcal{U}_s^{(t+1)}:|a\rangle|\bar{0}\rangle \mapsto |a\rangle|\mu_{t+1}(s,a)\rangle$ for all $a\in\mathcal{A}$, where $|\mu_{t+1}(s,a) - \sum_{s'\in\mathcal{S}}p(s'|s,a) u_{t+1}(s')| \leq \frac{\varepsilon}{2H}$ with high probability}

            \State \parbox[t]{\dimexpr\linewidth-\algorithmicindent-\algorithmicindent}{Use quantum maximum finding with unitary $\mathcal{U}_s^{(t+1)}$ (\Cref{fact:quantum_minimum_finding} with $\mathcal{U}_s^{(t+1)}$) to obtain $\widehat{u}_t(s)$ and $a_t(s)$ such that, with probability $1-\frac{\delta}{HS}$,
			\begin{align*}
				\widehat{u}_{t}(s) = \max_{a\in\mathcal{A}}\{r(s,a) + \mu_{t+1}(s,a)\} \quad\text{and}\quad
				 \pi_t(s) = \argmax_{a\in\mathcal{A}}\{r(s,a) + \mu_{t+1}(s,a)\}
			\end{align*}}\label{line:quantum_simple}

			\State $u_t(s) \gets \widehat{u}_t(s) - \frac{\varepsilon}{2H}$

		
            \EndFor

        \EndFor
    
    \State {\bfseries return} $\pi = (\pi_1,\dots,\pi_H)$
\end{algorithmic}
\end{algorithm}

Our second quantum algorithm (\Cref{algo:quantum_backward_recursion2}) is based on a simpler backward induction algorithm without variance reduction and total-variance techniques: we simply compute $u_t = \mathcal{L}u_{t+1}$ for $t=H,H-1,\dots,1$, i.e., in a backward fashion, using quantum subroutines. Although this will inevitably lead to a worse sample complexity on the horizon $H$, it is possible now to employ quantum minimum finding~\cite{durr1996quantum} together with quantum mean estimation, which brings down the sample complexity on the action space size from $O(A)$ down to $O(\sqrt{A})$.

\begin{theorem}\label{thr:quantum_finite-horizon2}
    Let $M = \langle \mathcal{S},\mathcal{A}, p, r,H\rangle$ be a finite-horizon MDP. Let $\varepsilon > 0$ and $\delta\in(0,1)$. {\rm \Cref{algo:quantum_backward_recursion2}} computes functions $\{u_t\}_{t\in[H]}\subset\mathscr{B}(\mathcal{S})$ and deterministic policy $\pi$ such that, with probability $1-\delta$,
    \begin{align*}
        V^\ast_t(s) - \varepsilon \leq u_t(s) \leq V^{\pi}_t(s) \leq  V^\ast_t(s) \qquad\forall (s,t)\in\mathcal{S}\times[H].
    \end{align*}
    The $\mathcal{Q}_p/\mathcal{Q}_p^\dagger$-query complexity is
    \begin{align*}
        O\bigg(\frac{H^{3}S\sqrt{A}}{\varepsilon}\log\left(\frac{HSA}{\delta} \right)\log\left(\frac{HS}{\delta} \right) \bigg).
    \end{align*}
\end{theorem}
\begin{proof}
    The proof is by induction on $t\in[H]$. We start by analysing the quantity $\widehat{u}_t$ in \Cref{line:quantum_simple} at time step $t\in[H]$. We claim that \cref{lem:quantum_mean_estimation} can be adapted to perform all of its steps in superposition without any need for intermediary measurements. This means that, given an initial state $\sum_{a\in\mathcal{A}}\alpha_a |s,a\rangle|\bar{0}\rangle$ for fixed $s\in\mathcal{S}$ and some $\{\alpha_a\}_{a\in\mathcal{A}}\subset\mathbb{C}$, \cref{lem:quantum_mean_estimation} can generate the mapping
    \begin{align*}
        \sum_{a\in\mathcal{A}}\alpha_a |s,a\rangle|\bar{0}\rangle \mapsto \sum_{a\in\mathcal{A}}\alpha_a |s,a\rangle|\mu_{t+1}(s,a)\rangle|\operatorname{garbage}(a)\rangle,
    \end{align*}
    where $\{|\operatorname{garbage}(a)\rangle\}_{a\in\mathcal{A}}$ are ``garbage'' unit complex vectors accumulated through the computation that we ignore and
    \begin{align*}
        \left| \mu_{t+1}(s, a) - \sum_{s'\in\mathcal{S}} p(s'|s,a) {u}_{t+1}(s') \right| \leq \frac{\varepsilon}{2H} \qquad \forall(s,a)\in\mathcal{S}\times\mathcal{A}.
    \end{align*}
    It is then possible to effectively construct a black-box unitary $\mathcal{U}_s^{(t)}:|a\rangle|\bar{0}\rangle \mapsto |a\rangle|{\mu}_{t+1}(s,a)\rangle$ (up to garbage states) using oracle $\mathcal{Q}_p$. The maximum over $a\in\mathcal{A}$ of $r(s,a) + \mu_{t+1}(s,a)$ can thus be found by using the quantum max-finding subroutine (\cref{fact:quantum_minimum_finding}) with unitary $\mathcal{U}_s^{(t)}$, leading to $|\widehat{u}_t(s) - (\mathcal{L}u_{t+1})(s)| \leq \frac{\varepsilon}{2H}$ with probability $1 - \frac{\delta}{HS}$. By defining $u_t(s) := \widehat{u}_t(s) - \frac{\varepsilon}{2H}$ for $s\in\mathcal{S}$, then
    \begin{align}\label{eq:simple_quantum_eq1}
        u_t(s) \leq (\mathcal{L}u_{t+1})(s) \qquad\text{and}\qquad
        u_t(s) \geq (\mathcal{L}u_{t+1})(s) - \frac{\varepsilon}{H}.
    \end{align}
    From here, $u_t(s) = (\mathcal{L}_{\pi_t}u_{t+1})(s) \leq (\mathcal{L}_{\pi_t}V_{t+1}^\pi)(s) = V_t^\pi$ by using the induction hypothesis to argue that $u_{t+1}(s) \leq V_{t+1}^\pi(s)$. To prove that $V_t^\ast(s) - \varepsilon \leq u_t(s)$, we follow the proof of \Cref{thr:quantum_finite-horizon},
    \begin{align*}
        V_t^\ast(s) - u_t(s) &\leq (\mathcal{L}V^\ast_{t+1})(s) - (\mathcal{L}u_{t+1})(s) + \frac{\varepsilon}{H} \leq \sum_{s'\in\mathcal{S}} p(s'|s, \pi_t^\ast(s)) \big(V^\ast_{t+1}(s') - u_{t+1}(s')\big) + \frac{\varepsilon}{H}.
    \end{align*}
    Solving the above recursion and using $\|\sum_{t'=t}^{H} \prod_{i=t}^{t'-1} P_{\pi^\ast_i}\|_\infty \leq H$, we get that $V_t^\ast(s) - u_t(s) \leq \varepsilon$, which concludes the proof of correctness. 

    We now analyse the success probability of \Cref{algo:quantum_backward_recursion2}. To do so, we must analyse how quantum oracles fail (see~\cite[Appendix~A]{chen2023quantum} for a similar argument as follows). Ideally, we would like to implement the unitary $\mathcal{U}_{s, {\rm ideal}}^{(t+1)}:|a\rangle|\bar{0}\rangle \mapsto |a\rangle|\mu_{t+1}(s,a)\rangle$ where $|\mu_{t+1}(s,a)\rangle$ contains the approximation $|\mu_{t+1}(s,a) - \sum_{s'\in\mathcal{S}}p(s'|s,a) u_{t+1}(s')| \leq \frac{\varepsilon}{2H}$. In practice, however, we implement the unitary $\mathcal{U}_{s}^{(t+1)}:|a\rangle|\bar{0}\rangle \mapsto |a\rangle(\sqrt{1-\delta_2}|\mu_{t+1}(s,a)\rangle + \sqrt{\delta_2}|\phi_a\rangle)$ for some $\delta_2 \in(0,1)$, where the register $|\mu_{t+1}(s,a)\rangle$ holds the desired approximation $\mu_{t+1}(s,a)$ and $|\phi_a\rangle$ is a normalised quantum state orthogonal to $|\mu_{t+1}(s,a)\rangle$. Notice that
    \begin{align*}
        \forall a\in\mathcal{A}: \quad \|(\mathcal{U}_{s, {\rm ideal}}^{(t+1)} - \mathcal{U}_{s}^{(t+1)})|a\rangle|\bar{0}\rangle\| = \sqrt{(1 - \sqrt{1-\delta_2})^2 + \delta_2} = \sqrt{2-2\sqrt{1-\delta_2}} \leq \sqrt{2\delta_2},
    \end{align*}
    using that $\sqrt{1-\delta_2} \geq 1 - \delta_2$. Since quantum minimum finding does not take into account the action of $\mathcal{U}_{s}^{(t+1)}$ onto states of the form $|a\rangle|\bar{0}^\perp\rangle$ for $|\bar{0}^\perp\rangle$ orthogonal to $|\bar{0}\rangle$, we can, without of loss of generality, assume that $\|\mathcal{U}_{s, {\rm ideal}}^{(t+1)} - \mathcal{U}_{s}^{(t+1)}\| \leq \sqrt{2\delta_2}$. The success probability of quantum minimum finding (\Cref{fact:quantum_minimum_finding}) is $1-\delta_1$ when employing $\mathcal{U}_{s, {\rm ideal}}^{(t+1)}$, for some $\delta_1\in(0,1)$. However, since it employs $\mathcal{U}_{s}^{(t+1)}$ instead, the success probability decreases by at most the spectral norm of the difference between the ``real'' and the ``ideal'' total unitaries. To be more precise, the ``ideal'' quantum minimum finding is a sequence of gates $W = U_1E_1U_2 E_2\cdots U_{N}E_N$, where $U_i \in \{\mathcal{U}_{s, {\rm ideal}}^{(t+1)},\mathcal{U}_{s, {\rm ideal}}^{(t+1)\dagger}\}$, $E_i$ is a circuit of elementary gates, and $N = c\sqrt{A}\log\frac{1}{\delta_1}$ with $c$ constant is the number of queries to $\mathcal{U}_{s, {\rm ideal}}^{(t+1)}$. The ``real'' implementation, on the other hand, is $\widetilde{W} = \widetilde{U}_1E_1\widetilde{U}_2 E_2\cdots \widetilde{U}_{N}E_N$, where $\widetilde{U}_i \in \{\mathcal{U}_{s}^{(t+1)},\mathcal{U}_{s}^{(t+1)\dagger}\}$. Then $\|W - \widetilde{W}\| \leq c\sqrt{A}\log\!\big(\frac{1}{\delta_1}\big)\|\mathcal{U}_{s, {\rm ideal}}^{(t+1)} - \mathcal{U}_{s}^{(t+1)}\| \leq c\sqrt{2\delta_2A}\log\frac{1}{\delta_1}$ and the failure probability is $\delta_1 + c\sqrt{2\delta_2A}\log\frac{1}{\delta_1}$. By taking $\delta_1 = O\big(\frac{\delta}{HS}\big)$ and $\delta_2 = O\big(\frac{\delta_1^2}{A\log^2(1/\delta_1)}\big)$, the failure probability in outputting $\max_{a\in\mathcal{A}}\{r(s,a) + \mu_{t+1}(s,a)\}$ is at most $\frac{\delta}{HS}$. By a usual union bound over all $s\in\mathcal{S}$ and $t\in[H]$, the failure probability is at most $\delta$.\footnote{\cite[Theorem~6]{wang2021quantum} employs a similar but slightly incorrect argument for the failure probability. More precisely, they ignore the failure probability $\delta_1$ coming from quantum minimum finding, or rather, assume it to be constant. Their runtime incorrectly ignores the factor $\log\frac{1}{\delta_1}$~\cite[Theorem~7]{wang2021quantum}.}
    
    Regarding the query complexity, for each $s\in\mathcal{S}$, one call to the unitary $\mathcal{U}_s^{(t+1)}$ uses $O\big(\frac{\|u_{t+1}\|_\infty}{\varepsilon/H}\log\frac{1}{\delta_2}\big) = O\big(\frac{H^2}{\varepsilon}\log\frac{HSA}{\delta}\big)$ queries to $\mathcal{Q}_p$, while quantum maximum finding makes $O\big(\sqrt{A}\log\frac{1}{\delta_1}\big) = O\big(\sqrt{A}\log\frac{HS}{\delta}\big)$ queries to $\mathcal{U}_s^{(t+1)}$. Summing over all $\mathcal{S}\times[H]$, the total query complexity is
    \[
        O\bigg(\frac{H^3S\sqrt{A}}{\varepsilon}\log\frac{HSA}{\delta}\log\frac{HS}{\delta}\bigg). \qedhere
    \]
\end{proof}

\subsection{Infinite-Horizon Discounted MDPs}
\label{sec:value_iteration}

We move on to finding approximate optimal policies for infinite-horizon undiscounted weakly communicating MDPs $\langle \mathcal{S},\mathcal{A},p,r\rangle$ such that $\operatorname{sp}(h^\ast) \leq \Lambda$. Our approach is the simple value iteration algorithm, which is quite similar to the backward induction algorithm covered in the last section: starting from some initial function $u_{0}\in\mathscr{B}(\mathcal{S})$, normally $u_{0} \equiv 0$, generate a sequence of functions $(u_{t})_{t\in\mathbb{N}}$ according to the update rule $u_{t+1} = \mathcal{L}u_{t}$. It is known that when $\operatorname{sp}(u_{t+1} - u_{t}) \leq \varepsilon$, the greedy policy with respect to $u_{t}$ is $\varepsilon$-optimal~\citep[Theorem~9.4.5]{puterman2014markov}. Here we are interested in robust versions of value iteration wherein errors in computing $\mathcal{L}{u}_{t}$ are taken into account. Approximate versions of value iteration have been studied and are used in various settings~\citep{farahmand2010error, mann2015approximate, ernst2005approximate, munos2007performance, de2000existence, van2006performance}, and here we consider a robust analogue of value iteration which differs from its standard implementation by generating a sequence of functions $({u}_{t})_{t\in\mathbb{N}}$ such that
\begin{align}\label{eq:approximate_value_iteration}
    \|{{u}}_{t+1} - \mathcal{L}{{u}}_{t}\|_\infty \leq \varepsilon_u  \quad\text{for a given}~\varepsilon_u \geq 0, \quad\text{where}~t\in\mathbb{N}.
\end{align}
The error $\varepsilon_u$ could come from, e.g., approximating $\sum_{s'\in\mathcal{S}}p(s'|s,a)u_{t}(s')$ or the maximum over $a\in\mathcal{A}$. In the following, we prove a few convergence results for the robust value iteration, starting by proving that any sequence of vectors  generated as in \eqref{eq:approximate_value_iteration} using some span contraction have bounded span.

\begin{lemma}\label{lem:stopping_criteria}
    Let $\epsilon \geq 0$ and $\mathcal{N}:\mathscr{B}(\mathcal{S})\to\mathscr{B}(\mathcal{S})$ be a $J$-stage $\nu$-span contraction for some $J\in\mathbb{N}$ and $\nu\in[0,1)$. Let $({u}_{t})_{t\in\mathbb{N}}$ be a sequence of functions such that $\|{u}_{t+1} - \mathcal{N} {u}_{t}\|_\infty \leq \epsilon$ $\forall t\in\mathbb{N}$. Then
    \begin{align*}
        \operatorname{sp}({u}_{t+1} -  {u}_{t}) \leq \nu^{\lfloor t/J\rfloor} \operatorname{sp}(\mathcal{N}u_{0} - u_{0}) + 4J\epsilon\frac{1-\nu^{\lfloor t/J\rfloor + 1}}{1-\nu} \qquad\text{for all}~ t\in\mathbb{N}.
    \end{align*}
\end{lemma}
\begin{proof}
    We shall prove by induction on $t\in\mathbb{N}$ that
    \begin{align*}
        \operatorname{sp}({u}_{Jt+k} -  {u}_{Jt+k-1}) \leq \nu^t\operatorname{sp}(\mathcal{N}{u}_{0} - {u}_{0}) + 4J\epsilon \frac{1-\nu^{t+1}}{1-\nu} \qquad\text{for all}~ k\in[J],
    \end{align*}
    which is equivalent to the original statement. The base case $t=0$ follows by
    \begin{align*}
        \operatorname{sp}(u_k - u_{k-1}) &\leq \sum_{i=0}^{k-1} \operatorname{sp}(\mathcal{N}^{i} u_{k-i} - \mathcal{N}^{i+1} u_{k-i-1}) + \sum_{i=0}^{k-2}\operatorname{sp}(\mathcal{N}^{i} u_{k-i-1} - \mathcal{N}^{i+1} u_{k-i-2}) + \operatorname{sp}(\mathcal{N}^k u_0 - \mathcal{N}^{k-1} u_0) \tag{triangle inequality} \\
        &\leq \sum_{i=0}^{k-1} \operatorname{sp}(u_{k-i} -  \mathcal{N} u_{k-i-1}) + \sum_{i=0}^{k-2} \operatorname{sp}(u_{k-i-1} -  \mathcal{N} u_{k-i-2}) + \operatorname{sp}(\mathcal{N} u_0 - u_0) \tag{contraction property and triangle inequality} \\
        &\leq 4J\epsilon + \operatorname{sp}(\mathcal{N} u_0 - u_0). \tag{$\operatorname{sp}({v}) \leq 2\|{v}\|_\infty$ for any ${v}\in\mathscr{B}(\mathcal{S})$ and $\|{u}_{t+1} - \mathcal{N} {u}_{t}\|_\infty \leq \epsilon$}
    \end{align*}
    For $t>0$,
    \begin{align*}
        \operatorname{sp}(u_{Jt+k} -  u_{Jt+k-1}) &\leq \sum_{i=0}^{J-1} \big(\operatorname{sp}(\mathcal{N}^{i} u_{Jt + k - i} - \mathcal{N}^{i+1} u_{Jt+k-i-1}) + \operatorname{sp}(\mathcal{N}^{i} u_{Jt +k - 1 - i} - \mathcal{N}^{i+1} u_{Jt+k-2-i})\big) \\
        &~~~~+ \operatorname{sp}(\mathcal{N}^J u_{J(t-1)+k} - \mathcal{N}^J u_{J(t-1)+k-1}) \tag{triangle inequality} \\
        &\leq \sum_{i=0}^{J-1} \big(\operatorname{sp}(u_{Jt + k - i} - \mathcal{N} u_{Jt+k-i-1}) + \operatorname{sp}( u_{Jt +k-1- i} - \mathcal{N} u_{Jt+k-2-i})\big) \\
        &~~~~+ \nu\operatorname{sp}(u_{J(t-1)+k} - u_{J(t-1)+k-1}) \tag{$J$-stage $\nu$-span contraction}\\
        &\leq 4J\epsilon + \nu \left(\nu^{t-1} \operatorname{sp}(\mathcal{N} u_{0} - u_{0}) + 4J\epsilon \frac{1-\nu^{t}}{1-\nu}\right)\tag{induction hypothesis}\\
        &= \nu^t\operatorname{sp}(\mathcal{N}{{u}}_{0} - {{u}}_{0}) + 4J\epsilon\frac{1-\nu^{t+1}}{1-\nu}. \tag*{\qedhere}
    \end{align*}
\end{proof}

If the error per iteration $\varepsilon_u$ is small enough, the above result tells us that after several iterations, robust value iteration obtains a pair of vectors ${{u}}_{t+1}$ and ${{u}}_{t}$ such that $\operatorname{sp}({{u}}_{t+1} - {{u}}_{t})$ is sufficiently small. It remains to show that, once a stopping criteria $\operatorname{sp}({{u}}_{t+1} - {{u}}_{t}) \leq \varepsilon_{s}$ for some $\varepsilon_{s} > 0$ is achieved, a corresponding policy $d_t^\infty$ associated with $u_{t+1}$ is close to optimal. The next result generalises~\citep[Theorem~8.5.6]{puterman2014markov}.

\begin{lemma}\label{lem:value_iteration_guarantees}
    Let $\varepsilon_u,\varepsilon_s>0$. Let $({u}_{t})_{t\in\mathbb{N}}$ be any sequence of vectors such that $\|{u}_{t+1} - \mathcal{L} {u}_{t}\|_\infty \leq \varepsilon_u$ for all $t\in\mathbb{N}$. Assume that $\operatorname{sp}({{u}}_{t+1} - {{u}}_{t}) \leq \varepsilon_s$ for some $\varepsilon_s > 0$ and $t\geq t_{\varepsilon}\in\mathbb{N}$. Let the deterministic decision rule $d_\varepsilon$ such that $\|u_{t_\varepsilon+1} - \mathcal{L}_{d_\varepsilon}u_{t_\varepsilon}\|_\infty \leq \varepsilon_u$. Define the quantity
    \begin{align*}
        g^{\varepsilon} := \frac{1}{2}\left(\max_{s\in\mathcal{S}}\{{u}_{t_{\varepsilon}+1}(s) - {u}_{t_{\varepsilon}}(s)\} + \min_{s\in\mathcal{S}}\{{u}_{t_{\varepsilon}+1}(s) - {u}_{t_{\varepsilon}}(s)\} \right).
    \end{align*}
    Then $\|g^{d_\varepsilon^\infty} - g^\varepsilon \mathbf{1}\|_\infty \leq \varepsilon_u + \frac{\varepsilon_s}{2}$ and $|g^\varepsilon - g^\ast| \leq \varepsilon_u + \frac{\varepsilon_s}{2}$, and so $d_\varepsilon^\infty$ is a $(2\varepsilon_u + \varepsilon_s)$-optimal policy. 
\end{lemma}
\begin{proof}
    It is possible to show, as a consequence of the optimality equations, that $g^{d_\varepsilon^\infty} = P^\ast_{d_\varepsilon}r_{d_\varepsilon}$ where $P^\ast_{d_\varepsilon} = \lim_{N\to\infty}\frac{1}{N}\sum_{t=1}^N P_{d_\varepsilon}^{t-1}$ (see~\cite[Theorem~8.2.6]{puterman2014markov} or \cite[Theorem~2.5]{saldi2017asymptotic}). Furthermore, since $P^\ast_{d_\varepsilon} P_{d_\varepsilon} = P_{d_\varepsilon}^\ast$,
    \begin{align*}
        {g}^{d_\varepsilon^\infty} &= P^\ast_{d_\varepsilon}(r_{d_\varepsilon} + P_{d_\varepsilon}u_{t_\varepsilon} - u_{t_\varepsilon}) = P^\ast_{d_\varepsilon}(\mathcal{L}_{d_\varepsilon}u_{t_\varepsilon} - u_{t_\varepsilon}).
    \end{align*}
    Since $\min_{s'\in\mathcal{S}}\{{u}_{t_\varepsilon+1}(s') - {u}_{t_\varepsilon}(s')\} \leq {{u}}_{t_\varepsilon+1}(s) - {{u}}_{t_\varepsilon}(s) \leq \max_{s'\in\mathcal{S}}\{{u}_{t_\varepsilon+1}(s') - {u}_{t_\varepsilon+1}(s')\} $ for all $s\in\mathcal{S}$, then
    \begin{align*}
        \min_{s'\in\mathcal{S}}\{{u}_{t_\varepsilon+1}(s') - {u}_{t_\varepsilon}(s')\}\mathbf{1} \leq P^\ast_{d_\varepsilon}\big(u_{t_\varepsilon+1} - u_{t_\varepsilon}\big) \leq \max_{s'\in\mathcal{S}}\{{u}_{t_\varepsilon+1}(s') - {u}_{t_\varepsilon+1}(s')\}\mathbf{1},
    \end{align*}
    from which it follows that
    \begin{align*}
        \big\| P^\ast_{d_\varepsilon}\big( {{u}}_{t_\varepsilon+1} - {{u}}_{t_\varepsilon}\big) - g^\varepsilon \mathbf{1} \big\|_\infty \leq \frac{1}{2} \operatorname{sp}({{u}}_{t_\varepsilon+1} - {{u}}_{t_\varepsilon}) \leq \frac{\varepsilon_s}{2}
        \implies \|{g}^{d_\varepsilon^\infty} - g^\varepsilon \mathbf{1}\|_\infty \leq \varepsilon_u + \frac{\varepsilon_s}{2},
    \end{align*}
    using that $\|u_{t_\varepsilon+1} - \mathcal{L}_{d_\varepsilon}u_{t_\varepsilon}\|_\infty \leq \varepsilon_u$. By repeating the same procedure but with the optimal policy $(d^\ast)^\infty$ instead of $d_\varepsilon^\infty$, then $|g^\ast - g^\varepsilon| \leq \varepsilon_u+ \frac{\varepsilon_s}{2}$ (this time using that $\|u_{t+1} - \mathcal{L}u_{t}\|_\infty \leq \varepsilon_u$).
\end{proof}

In the same way that we quantised the finite-horizon classical backward iteration algorithm by approximating the quantities $\sum_{s'\in\mathcal{S}}p(s'|s,a)u_{t}(s')$ using quantum mean estimation, a similar procedure can be performed for the classical value iteration algorithm on infinite-horizon MDPs. 

\begin{algorithm}[t!]
    \caption{Quantum value iteration algorithm}
    \label{algo:quantum_extended_value_iteration}
    \begin{algorithmic}[1]  
    \Require Finite state space $\mathcal{S}$ and action space $\mathcal{A}$, quantum sampling access to probability kernels $p$, failure probability $\delta\in(0, 1)$, parameters $J\in\mathbb{N}$, $\nu\in[0,1)$, errors $\varepsilon_u, \varepsilon > 0$ such that $\varepsilon_u = \frac{1}{4J}(1-\nu)\varepsilon$.

    \Ensure $2\varepsilon$-optimal stationary deterministic policy $\widetilde{d}^\infty$. 

    \State $t\gets 1$ and initialise ${u}_{0} \gets {0}$ 
    \State $u_{1}(s) \gets \max_{a\in\mathcal{A}}\{r(s,a)\}$ for all $s\in\mathcal{S}$ with probability $1-\frac{6\delta}{\pi^2 S}$ (\Cref{fact:quantum_minimum_finding})

    \While {$\operatorname{sp}({{u}}_{t} - {{u}}_{t-1}) > \frac{3\varepsilon}{2}$}
    
    \State Build quantum access to ${{u}}_{t}$
    \For{$s\in\mathcal{S}$}
        \State \parbox[t]{\dimexpr\linewidth-\algorithmicindent-\algorithmicindent}{Construct unitary $\mathcal{U}^{(t)}_s:|a\rangle|\bar{0}\rangle \mapsto |a\rangle|{\mu}_{t}(s,a)\rangle$ for all $a\in\mathcal{A}$ using \cref{lem:quantum_mean_estimation}, where $|{\mu}_{t}(s,a) - \sum_{s'\in\mathcal{S}} p(s'|s,a) {u}_{t}(s') | \leq \frac{\varepsilon_u}{2}$ with high probability}
        
        \State \parbox[t]{\dimexpr\linewidth-\algorithmicindent-\algorithmicindent}{Use quantum maximum finding with unitary $\mathcal{U}_s^{(t)}$ (\cref{fact:quantum_minimum_finding} with $\mathcal{U}^{(t)}_s$) to get $\widetilde{u}_{t+1}(s)$ and $a_{t+1}(s)$ such that, with probability $1-\frac{6\delta}{\pi^2 t^2 S}$,
        \begin{align*}
            \widetilde{u}_{t+1}(s) = \max_{a\in\mathcal{A}}\{r(s,a) + \mu_{t+1}(s,a)\} \quad\text{and}\quad a_{t+1}(s) = \argmax_{a\in\mathcal{A}}\{r(s,a) + \mu_{t+1}(s,a)\}
        \end{align*}}
    \EndFor

    \State For all $s\in\mathcal{S}$,
    \begin{align*}
        u_{t+1}(s) \gets \begin{dcases}
            \max\left\{\max_{s'\in\mathcal{S}}\widetilde{u}_{t+1}(s') - \frac{\varepsilon_u}{2}, \min_{s'\in\mathcal{S}}\widetilde{u}_{t+1}(s')\right\} &\text{if}~\widetilde{u}_{t+1}(s) \geq \max_{s'\in\mathcal{S}}\widetilde{u}_{t+1}(s') - \frac{\varepsilon_u}{2},\\
            \min\left\{\min_{s'\in\mathcal{S}}\widetilde{u}_{t+1}(s') + \frac{\varepsilon_u}{2}, \max_{s'\in\mathcal{S}}\widetilde{u}_{t+1}(s')\right\} &\text{if}~ \widetilde{u}_{t+1}(s) \leq \min_{s'\in\mathcal{S}}\widetilde{u}_{t+1}(s') + \frac{\varepsilon_u}{2},\\
            \widetilde{u}_{t+1}(s) &\text{otherwise}.
        \end{dcases}
    \end{align*}

    \State $t \gets t+1$

    \EndWhile
    
    \State {\bfseries return} deterministic decision rule $\widetilde{d}$ where $\widetilde{d}(s) := a_t(s)$ for all $s\in\mathcal{S}$
\end{algorithmic}
\end{algorithm}

\begin{theorem}\label{thr:quantum_scopt_algorithm}
    Let $M = \langle \mathcal{S},\mathcal{A},p,r\rangle$ be an infinite-horizon undiscounted weakly communicating MDP with $\operatorname{sp}({h}^\ast) \leq \Lambda$. Assume the optimal Bellman operator $\mathcal{L}$ of $M$ is a $J$-stage $\nu$-span contraction for $J\in\mathbb{N}$ and $\nu\in[0,1)$. Let $\delta\in (0, 1)$ and $\varepsilon > 0$. {\rm \cref{algo:quantum_extended_value_iteration}} outputs an $2\varepsilon$-optimal policy $\widetilde{d}^\infty$ and $g^\varepsilon$ such that $|g^\varepsilon - g^\ast| < \varepsilon$ with probability $1-\delta$. Its $\mathcal{Q}_p/\mathcal{Q}_p^\dagger$-query complexity is (up to $\poly\log\log$ factors in $1/\varepsilon$ and $\nu$)
    \begin{align*}
        \widetilde{O}\left(\frac{\Lambda S\sqrt{A}}{(1-\nu)\varepsilon}\frac{J^2\log\frac{1}{J\varepsilon}}{\log\frac{1}{\nu}}\log\frac{JSA}{\delta}\log\frac{JS}{\delta}\right).
    \end{align*}
    %
\end{theorem}
\begin{proof}
    Let $\varepsilon_u> 0$ such that $\varepsilon_u := \frac{1}{4J}(1-\nu)\varepsilon$. We start by proving that \Cref{algo:quantum_extended_value_iteration} generates a sequence of functions $(u_t)_{t\in\mathbb{N}}$ such that $\|u_{t+1} - \mathcal{L}u_t\|_\infty \leq \varepsilon_u$ and $\operatorname{sp}(u_{t+1}) \leq \operatorname{sp}(\mathcal{L}u_t)$. The former inequality comes from nesting quantum mean estimation (\cref{lem:quantum_mean_estimation}) within quantum minimum finding (\cref{fact:quantum_minimum_finding}) just like in \Cref{thr:quantum_finite-horizon} in order to obtain $\widetilde{u}_{t+1}(s) = \max_{a\in\mathcal{A}}\{r(s,a) + \mu_{t+1}(s,a)\}$ and $a_{t+1}(s) = \argmax_{a\in\mathcal{A}}\{r(s,a) + \mu_{t+1}(s,a)\}$ such that $\|\widetilde{u}_{t+1} - \mathcal{L} {{u}}_{t}\|_\infty \leq \frac{\varepsilon_u}{2}$. By subtracting $\frac{\varepsilon_u}{2}$ from the larger entries of $\widetilde{u}_{t+1}$ and adding $\frac{\varepsilon_u}{2}$ to the smaller entries of $\widetilde{u}_{t+1}$, we define, for all $s\in\mathcal{S}$,
    \begin{align*}
        u_{t+1}(s) \!=\! \begin{cases}
            \max\{\max_{s'\in\mathcal{S}}\widetilde{u}_{t+1}(s') - \frac{\varepsilon_u}{2}, \min_{s'\in\mathcal{S}}\widetilde{u}_{t+1}(s')\} &\text{if}~\widetilde{u}_{t+1}(s) \geq \max_{s'\in\mathcal{S}}\widetilde{u}_{t+1}(s') - \frac{\varepsilon_u}{2},\\
            \min\{\min_{s'\in\mathcal{S}}\widetilde{u}_{t+1}(s') + \frac{\varepsilon_u}{2}, \max_{s'\in\mathcal{S}}\widetilde{u}_{t+1}(s')\} &\text{if}~ \widetilde{u}_{t+1}(s) \leq \min_{s'\in\mathcal{S}}\widetilde{u}_{t+1}(s') + \frac{\varepsilon_u}{2},\\
            \widetilde{u}_{t+1}(s) &\text{otherwise}.
        \end{cases}
    \end{align*}
    In other words, the smaller entries of $\widetilde{u}_{t+1}$ are increases by $\frac{\varepsilon_u}{2}$, while its larger entries are decreased by $\frac{\varepsilon_u}{2}$. This means that $\operatorname{sp}(u_{t+1}) \leq \operatorname{sp}(\mathcal{L}u_t)$ and that $\|u_{t+1} - \mathcal{L}u_t\|_\infty \leq \varepsilon_u$ as required. Note also that the generated functions $(u_t)_{t\in\mathbb{N}}$ satisfy
    \begin{align*}
        \operatorname{sp}(u_t) \leq \operatorname{sp}(\mathcal{L}^t 0) \leq \operatorname{sp}(\mathcal{L}^t h^\ast) + \operatorname{sp}(\mathcal{L}^t 0 - \mathcal{L}^t h^\ast) \leq 2\operatorname{sp}(h^\ast) \leq 2\Lambda, \tag{$\mathcal{L}$ is non-expansive}
    \end{align*}
    where we used that 
    \begin{align*}
        \operatorname{sp}(\mathcal{L}^t h^\ast) \leq \operatorname{sp}(h^\ast) + \sum_{i=1}^t \operatorname{sp}(\mathcal{L}^i h^\ast - \mathcal{L}^{i-1}h^\ast) \leq \operatorname{sp}(h^\ast) + (t-1) \operatorname{sp}(\mathcal{L} h^\ast - h^\ast) = \operatorname{sp}(h^\ast). \tag{$\mathcal{L}$ is non-expansive and $\mathcal{L}h^\ast = h^\ast + g^\ast\mathbf{1}$}
    \end{align*}

    By \Cref{lem:stopping_criteria}, $\operatorname{sp}(u_{t+1} - u_t) \leq \frac{6J\varepsilon_u}{1-\nu}$ if $t \geq t^\ast = O\big(J\frac{\log(1/J\varepsilon_u)}{\log(1/\nu)}\big)$. We employ \cref{lem:value_iteration_guarantees} to argue that, once $\operatorname{sp}({{u}}_{t+1} - {{u}}_{t}) \leq \frac{6J\varepsilon_u}{1-\nu}$ for $t\geq t^\ast$, the output
    \begin{align*}
        g^\varepsilon = \frac{1}{2}\left(\max_{s\in\mathcal{S}}\{{u}_{t^\ast + 1}(s) - {u}_{t^\ast}(s)\} + \min_{s\in\mathcal{S}}\{{u}_{t^\ast + 1}(s) - {u}_{t^\ast}(s)\}\right)
    \end{align*}
    is such that $|g^\varepsilon - g^\ast| \leq \varepsilon_u + \frac{3J\varepsilon_u}{1-\nu} < \varepsilon$, while the deterministic policy $\widetilde{d}^\infty$ with $\widetilde{d}(s) := \argmax_{a\in\mathcal{A}}\{u_{t^\ast}(s)\}$ for all $s\in\mathcal{S}$ is $2\varepsilon$-optimal.

    The analysis of the success probability of \Cref{algo:quantum_extended_value_iteration} is very similar to the one in \Cref{thr:quantum_finite-horizon2}. Taking into account the oracle failure probabilities, outputting $\max_{a\in\mathcal{A}}\{r(s,a) + \mu_{t+1}(s,a)\}$ fails with probability at most $\frac{6\delta}{\pi^2 t^2 S}$. By a usual union bound over all $s\in\mathcal{S}$ and $t\in\mathbb{N}$, the failure probability is at most $\delta$, since $\sum_{t=1}^\infty \frac{1}{t^2} = \frac{\pi^2}{6}$.
    
    Finally, for the query complexity of \Cref{algo:quantum_extended_value_iteration}, we start from obtaining ${{u}}_{t+1}$ given ${{u}}_{t}$. For every $s\in\mathcal S$, each call to the unitary $\mathcal{U}_s^{(t)}$ uses $O\big(\frac{\operatorname{sp}(u_t)}{\varepsilon_u}\log\frac{t SA}{\delta}\big)$ queries to $\mathcal{Q}_p$, while quantum maximum finding makes $O\big(\sqrt{A}\log\frac{t|\mathcal{S}_n|}{\delta}\big)$ queries to $\mathcal{U}_s^{(t)}$. Thus the query complexity of computing $u_{t+1}(s)$ (and $a_{t+1}(s)$) is $O\big(\frac{\Lambda\sqrt{A}}{\varepsilon_u}\log\frac{tSA}{\delta}\log\frac{t S}{\delta}\big)$, already using that $\operatorname{sp}(u_t) \leq 2\Lambda$. Summing over all $s\in\mathcal{S}$ and $t\leq t^\ast$, the total query complexity of \Cref{algo:quantum_extended_value_iteration} is simply $O\big(J t^\ast \frac{\Lambda S\sqrt{A}}{(1-\nu)\varepsilon}\log\frac{t^\ast SA}{\delta}\log\frac{t^\ast S}{\delta}\big)$.
\end{proof}

\end{document}